\documentclass[preprint]{article}

\usepackage{neurips_2020}

\usepackage[utf8]{inputenc} 
\usepackage[T1]{fontenc}    
\usepackage{hyperref}       
\usepackage{url}            
\usepackage{stix}            
\usepackage{booktabs}       
\usepackage{amsfonts}       
\usepackage{nicefrac}       
\usepackage{microtype}      
\usepackage{dsfont}

\title{Beyond Perturbations: Learning Guarantees with Arbitrary Adversarial Test Examples}

\author{%
 Shafi Goldwasser\thanks{Author order is alphabetical.}
 \\
 UC Berkeley  and MIT
 \AND
 Adam Tauman Kalai 
 \\
 Microsoft Research
\AND
 Yael Tauman Kalai 
 \\
 Microsoft Research and MIT
\AND
 Omar Montasser 
 \\
 TTI Chicago
}

\usepackage{amsfonts,amssymb,amsmath,amsthm,mathtools} 
\usepackage{graphicx,microtype,latexsym,lpic,bm,xspace,booktabs,wrapfig,array}

\usepackage[bf]{caption}

\usepackage{xcolor}
\usepackage{algorithmicx,algorithm}
\usepackage[noend]{algpseudocode}
\usepackage{thmtools,thm-restate}
\usepackage{tocloft}

\usepackage{dsfont}
\usepackage[greek,english]{babel}
\usepackage{caption}
\usepackage{subcaption}

\usepackage{times, enumerate, amssymb, mathtools}
\usepackage[colorinlistoftodos,prependcaption,textsize=tiny]{todonotes}

\newcommand{\bc}{\mathbf{c}}

\newcommand{\bw}{\mathbf{w}}
\newcommand{\bx}{\mathbf{x}}
\newcommand{\by}{\mathbf{y}}
\newcommand{\bz}{\mathbf{z}}

\newcommand{\tx}{{\tilde{x}}}
\newcommand{\tbx}{{\tilde{\mathbf{x}}}}
\newcommand{\btx}{{\tilde{\mathbf{x}}}}
\newcommand{\tmu}{{\tilde{\mu}}}

\newcommand{\barS}{\bar{S}}

\newcommand{\TV}{\mathsf{TV}}

\newcommand{\ERM}{\mathsf{ERM}}

\newcommand{\wwtp}{\mathsf{Rejectron}}
\newcommand{\wtp}{\mathsf{URejectron}}
\newcommand{\dis}{\mathsf{dis}}
\newcommand{\Dis}{\mathsf{Dis}}
\newcommand{\DIS}{\mathsf{DIS}}
\newcommand{\barf}{\bar{f}}

\newcommand\functions{C}

\newcommand\defeq{\coloneqq}

\newcommand\eps{\epsilon}

\newcommand\cA{\mathcal{A}}

\newcommand\cD{\mathcal{D}}

\newcommand\reals{\mathbb{R}}
\newcommand\err{\operatorname{err}}

\DeclareMathOperator*{\E}{\mathbb{E}}
\newcommand\nats{\mathbb{N}}
\newcommand\qnat{Q_{\text{nat}}}
\newcommand\qadv{Q_{\text{adv}}}

\newcommand\question{{\textcolor{gray}\vrectangleblack}}
\newcommand\rej{\operatorname{\question}}

\usepackage{bbm}

\usepackage{cleveref}
\newtheorem{theorem}{Theorem}[section]
\newtheorem{definition}[theorem]{Definition}

\newtheorem{lemma}[theorem]{Lemma}

\newtheorem{corollary}[theorem]{Corollary}

\newtheorem{claim}[theorem]{Claim}

\begin{document}

\maketitle

\begin{abstract}
We present a transductive learning algorithm that takes as input training examples from a distribution $P$ and {\em arbitrary} (unlabeled) test examples, possibly chosen by an adversary. This is unlike prior work that assumes that test examples are small perturbations of $P$. Our algorithm outputs a \emph{selective classifier}, which abstains from predicting on some examples. By considering selective transductive learning, we give the first nontrivial guarantees for learning classes of bounded VC dimension with arbitrary train and test distributions---no prior guarantees were known even for simple classes of functions such as intervals on the line. In particular, for any function in a class $C$ of bounded VC dimension, we guarantee a low test error rate and a low rejection rate \emph{with respect to $P$}. Our algorithm is efficient given an Empirical Risk Minimizer (ERM) for $C$. Our guarantees hold even for test examples chosen by an unbounded white-box adversary. We also give guarantees for generalization, agnostic, and unsupervised settings.
\end{abstract}

\section{Introduction}
Consider binary classification where test examples are not from the training distribution. Specifically, consider learning a binary function $f: X \rightarrow \{0,1\}$ where training examples are assumed to be iid from a distribution $P$ over $X$, while the test examples are \textit{arbitrary}. This includes both the possibility that test examples are chosen by an adversary or that they are drawn from a distribution $Q\neq P$ (sometimes called ``covariate shift''). For a disturbing example of covariate shift, consider learning to classify abnormal lung scans. A system trained on scans prior to 2019 may miss abnormalities due to COVID-19 since there were none in the training data. As a troubling adversarial example, consider explicit content detectors which are trained to classify normal vs.\
explicit images. Adversarial spammers synthesize endless variations of explicit images that evade these detectors for purposes such as advertising and phishing \citep{Yuan2019StealthyPU}. 

A recent line of work on adversarial learning has designed algorithms that are robust to  imperceptible perturbations. However, perturbations do not cover all types of test examples. In the explicit image detection example, \cite{Yuan2019StealthyPU} find adversaries using conspicuous image distortion techniques (e.g., overlaying a large colored rectangle on an image) rather than imperceptible perturbations. In the lung scan example, \cite{fang2020sensitivity} find noticeable signs of COVID in many scans.

In general, there are several reasons why learning with arbitrary test examples is actually impossible. First of all, one may not be able to predict the labels of test examples that are far from training examples, as illustrated by the examples in group (1) of Figure \ref{fig:example}. 
Secondly, as illustrated by group (2), given any classifier $h$, an adversary or test distribution $Q$ may concentrate on or near an error. High error rates are thus unavoidable since an adversary can simply repeat any single erroneous example they can find. This could also arise naturally, as in the COVID example, if $Q$ contains a  concentration of new examples near one another--individually they appear ``normal'' (but are suspicious as a group). This is true even under the standard \textit{realizable} assumption that the target function $f \in C$ is in a known class $C$ of bounded VC dimension $d=\mathrm{VC}(C)$. 

As we now argue, learning with arbitrary test examples requires \textit{selective classifiers} and \textit{transductive learning},  which have each been independently studied extensively. We refer to the combination as classification with \textit{redaction}, a term which refers to the removal/obscuring of certain information when documents are released. A \textit{selective classifier} (SC) is one which is allowed to abstain from predicting on some examples. In particular, it specifies both a classifier $h$ and a subset $S \subseteq X$ of examples to classify, and rejects the rest. Equivalently, one can think of a SC as $h|_S:X \rightarrow \{0,1,\question\}$ where $\question$ indicates $x \not\in S$, abstinence.
\[h|_S(x)\defeq \begin{cases}h(x)&\text{if }x \in S\\\question & \text{if } x \not\in S.\end{cases}\]
We say the learner \textit{classifies} $x$ if $x \in S$ and otherwise it rejects $x$. Following standard terminology, if $x\notin S$ (i.e., $h|_S(x)=\question$) we say the classifier \textit{rejects} $x$ (the term is not meant to indicate anything negative about the example $x$ but merely that its classification may be unreliable). We sat that $h|_S$ \textit{misclassifies} or \textit{errs} on $x$ if $h|_S(x)=1-f(x)$. There is a long literature on SCs, starting with the work of \cite{chow1957optimum} on character recognition. In standard classification, \textit{transductive learning} refers to the simple learning setting where the goal is to classify a given unlabeled test set that is presented together with the training examples \citep[see e.g.,][]{vapnik1998statistical}. 
We will also consider the generalization error of the learned classifier.

This raises the question: \textit{When are unlabeled test examples available in advance?} In some applications, test examples are classified all at once (or in batches). Otherwise, redaction can also be beneficial {\em in retrospect}. For instance, even if image classifications are necessary immediately, an offensive image detector may be run daily with rejections flagged for inspection; and images may later be blocked if they are deemed offensive. Similarly, if a group of unusual lung scans showing COVID were detected after a period of time, the recognition of the new disease could be valuable even in hindsight. Furthermore, in some applications, one cannot simply label a sample of test examples. For instance, in learning to classify messages on an online platform, test data may contain both public and private data while training data may consist only of public messages. Due to privacy concerns, labeling data from the actual test distribution may be prohibited. 

It is clear that a SC is necessary to guarantee few test misclassifications, e.g., if $P$ is concentrated on a single point $x$, rejection is necessary to guarantee few errors on arbitrary test points. However, no prior guarantees (even statistical guarantees) were known even for learning elementary classes such as intervals or halfspaces with arbitrary $P\neq Q$. This is because learning such classes is impossible without unlabeled examples. 

To illustrate how redaction (transductive SC) is useful, consider learning an interval $[a,b]$ on $X=\reals$ with arbitrary $P\neq Q$. This is illustrated below with (blue) dots indicating test examples:
\[\includegraphics[width=\textwidth]{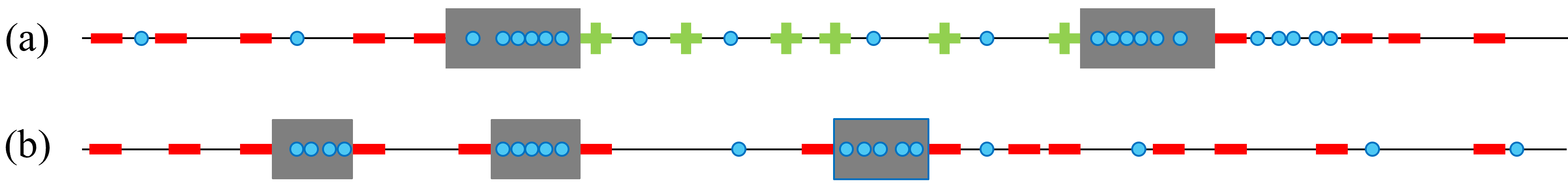}\]
With positive training examples as in (a), one can guarantee 0 test errors by rejecting the two (grey) regions adjacent to the positive examples. When there are no positive training examples,\footnote{Learning with an all-negative training set (trivial in standard learning) is a useful ``anomaly detection'' setting in adversarial learning, e.g., when one aims to classify illegal images without any illegal examples at train time or abnormal scans not present at train time.} as in (b), one can guarantee $\leq k$ test errors by rejecting any region with $>k$ test examples and no training examples; and predicting negative elsewhere. Of course, one can guarantee 0 errors by rejecting everywhere, but that would mean rejecting even future examples distributed like $P$. While our error objective will be an $\eps$ test error rate, our rejection objective will be more subtle since we cannot absolutely bound the test rejection rate. Indeed, as illustrated above, in some cases one should reject many test examples. 

Note that our redaction model assumes that the target function $f$ remains the same at train and test times. This assumption holds in several (but not all) applications of interest. For instance, in explicit image detection, U.S.\ laws regarding what constitutes an illegal image are based solely on the image $x$ itself \citep{congress_1996}. Of course, if laws change between train and test time, then $f$ itself may change. \emph{Label shift} problems where $f$ changes from train to test is also important but not addressed here. Our focus is primarily the well-studied realizable setting, where $f \in C$, though we analyze an agnostic setting as well.

\paragraph{A note of caution.} Inequities may be caused by using training data that differs from the test distribution on which the classifier is used. For instance, in classifying a person's gender from a facial image, \cite{gendershades} have demonstrated that commercial classifiers are highly inaccurate on dark-skinned faces, likely because they were trained on light-skinned faces. In such cases, it is preferable to collect a more diverse training sample even if it comes at greater expense, or in some cases to abstain from using machine learning altogether. In such cases, $PQ$ learning should \textit{not} be used, as an unbalanced distribution of rejections can also be harmful.\footnote{We are grateful to an anonymous reviewer who pointed out that gender classification is an example of when \textit{not} to use $PQ$ learning.}

\subsection{Redaction model and guarantees}
Our goal is to learn a target function $f\in C$ of VC dimension $d$ with training distribution $P$ over $X$. In the redaction model, the learner first chooses $h\in C$ based on $n$ iid training examples $\bx \sim X^n$ and their labels $f(\bx)=\bigl(f(x_1),f(x_2),\ldots, f(x_n)\bigr)\in \{0,1\}^n$. (In other words, it trains a standard binary classifier.) Next, a ``white box'' \textit{adversary} selects $n$ arbitrary test examples $\tbx \in X^n$ based on all information including $\bx, f, h, P$ and the learning algorithm. 
Using the unlabeled test examples (and the labeled training examples), the learner finally outputs $S \subseteq X$. Errors are those test examples in $S$ that were misclassified, i.e., $h|_S(x)=1-f(x)$. 

Rather than jumping straight into the transductive setting, we first describe the simpler generalization setting. We define the $PQ$ model in which $\tbx \sim Q^n$ are drawn iid by \textit{nature}, for an arbitrary distribution $Q$. While it will be easier to quantify generalization error and rejections in this simpler model, the $PQ$ model does not permit a white-box adversary to choose test examples based on $h$. To measure performance here, define rejection and error rates for distribution $D$
, respectively:
\begin{align}
\rej_D(S) &\defeq \Pr_{x\sim D}[x \not\in S] 
\label{eq:def1}\\
\err_D(h|_S) &\defeq \Pr_{x\sim D}[h(x)\neq f(x) \wedge x \in S] 
\label{eq:def2}
\end{align}
We write $\rej_D$ and $\err_D$ when $h$ and $S$ are clear from context.
We extend the definition of PAC learning to $P\neq Q$ as follows:
\begin{definition}[PQ learning]\label{def:pq}
Learner $L$ $(\eps, \delta, n)$-PQ-learns $C$ if for any distributions $P, Q$ over $X$ and any $f \in C$, its output $h|_S=L(\bx, f(\bx), \tbx)$ satisfies
\[\Pr_{\bx \sim P^n, \tbx \sim Q^n}\left[\rej_P+\err_Q \leq \eps \right]\geq 1-\delta.\]
$L$ PQ-learns $C$ if $L$ runs in polynomial time and if there is a polynomial $p$ such that $L$ $(\eps, \delta, n)$-PQ-learns $C$ for every $\eps, \delta>0, n \geq p(1/\eps, 1/\delta)$.
\end{definition}
Now, at first it may seem strange that the definition
bounds $\rej_P$ rather than $\rej_Q$, but as mentioned $\rej_Q$ cannot be bound absolutely. Instead, it can be bound relative to $\rej_P$ and the \textit{total variation distance} (also called statistical distance) $|P-Q|_{\TV} \in [0,1]$, as follows:
\[\rej_Q \leq \rej_P + |P-Q|_{\TV}.\]
This new perspective, of bounding the rejection probability of~$P$, as opposed to~$Q$, facilitates the analysis. Of course when $P=Q$, $|P-Q|_\TV=0$ and $\rej_Q=\rej_P$, and when $P$ and $Q$ have disjoint supports (no overlap), then $|P-Q|_\TV=1$ and the above bound is vacuous. We also discuss tighter bounds relating $\rej_Q$ to $\rej_P$.

We provide two redactive learning algorithms: a supervised algorithm called $\wwtp$, and an unsupervised algorithm $\wtp$.
$\wwtp$ takes as input~$n$ labeled training data $(\bx,\by)\in X^n\times \{0,1\}^n$ and~$n$ test data $\tbx\in X^n$ (and an error parameter~$\epsilon$).  It can be implemented efficiently using any $\ERM_C$ oracle that outputs a function $c \in C$ of minimal error on any given set of labeled examples. It is formally presented in Figure \ref{fig:algs}. At a high level, it chooses $h=\ERM(\bx,\by)$ and chooses $S$ in an iterative manner.  It starts with $S=X$ and then iteratively chooses $c\in C$ that disagrees significantly with $h|_S$ on $\tbx$ but agrees with $h|_S$ on $\bx$; it then rejects all $x$'s such that $c(x)\neq h(x)$.  As we show in Lemma~\ref{lem:eff}, choosing $c$ can be done efficiently given oracle access to $\ERM_C$.

\Cref{wow:real-gen} shows that $\wwtp$ PQ-learns any class $C$ of bounded VC dimension $d$, specifically with $\eps = \tilde{O}(\sqrt{d/n})$. (The $\tilde{O}$ notation hides logarithmic factors including the dependence on the failure probability $\delta$.) This is worse than the standard $\eps=\tilde{O}(d/n)$ bound of supervised learning when $P=Q$, though Theorem~\ref{thm:lower-gen} shows this is necessary with an $\Omega(\sqrt{d/n})$ lower-bound for $P\neq Q$.

Our unsupervised learning algorithm $\wtp$, formally presented in Figure~\ref{fig:nalgs}, computes $S$ only from unlabeled training and test examples, and has similar guarantees (\Cref{cor:wtp-real-trans}). The algorithm tries to distinguish training and test examples and then rejects whatever is almost surely a test example.
More specifically, as above, it chooses $S$ in an iterative manner, starting with $S=X$.  It (iteratively) chooses {\em two} functions $c,c'\in C$ such that $c|_S$ and $c'|_S$ have high disagreement on $\tbx$ and low disagreement on~$\bx$, and rejects all $x$'s on which $c|_S,c'|_S$ disagree.  As we show in Lemma~\ref{lem:eff-wtp}, choosing $c$ and $c'$ can be done efficiently given a (stronger) $\ERM_{\DIS}$ oracle for the class $\DIS$ of disagreements between $c,c' \in C$.
We emphasize that $\wtp$ can also be used for multi-class learning as it does not use training labels, and can be paired with any classifier trained separately. This advantage of $\wtp$ over $\wwtp$ comes at the cost of requiring a stronger base classifier to be used for $\ERM$, and may lead to examples being unnecessarily rejected.

In Figure \ref{fig:example} we illustrate our algorithms for the class $C$ of halfspaces. A natural idea would be to train a halfspace to distinguish unlabeled training and test examples---intuitively, one can safely reject anything that is clearly distinguishable as test without increasing $\rej_P$. However, this on its own is insufficient.
See for example group (2) of examples in Figure \ref{fig:example}, which cannot be distinguished from training data by a halfspace. This is precisely why having test examples is absolutely necessary.  Indeed, it allows us to use an ERM oracle to $C$ to PQ-learn $C$.

We also present:
\paragraph{Transductive analysis}
A similar analysis of $\wwtp$ in a transductive setting gives error and rejection bounds directly on the test examples. The bounds here are with respect to a stronger white-box adversary who need not even choose a test set $\tbx$ iid from a distribution.
Such an adversary chooses the test set with knowledge of $P, f, h$ and $\bx$. In particular,  first $h$ is chosen based on $\bx$ and $\by$; then the adversary chooses the test set $\tbx$ based on all available information; and finally, $S$ is chosen. We introduce a novel notion of \textit{false rejection}, where we reject a test example that was in fact chosen from $P$ and not modified by an adversary. \Cref{wow:real-trans} gives bounds that are similar in spirit to \Cref{wow:real-gen} but for the harsher transductive setting.

\paragraph{Agnostic bounds}
Thus far, we have considered the realizable setting where the target $f \in C$. In agnostic learning (\cite{Kearns92towardefficient}), there is an arbitrary distribution $\mu$ over $X \times \{0,1\}$ and the goal is to learn a classifier that is nearly as accurate as the best classifier in $C$. In our setting, we assume that there is a known $\eta\geq 0$ such that the train and test distributions $\mu$ and $\tilde{\mu}$ over $X \times \{0,1\}$ satisfy that  there is some function $f \in C$ that has error at most $\eta$ with respect to both $\mu$ and $\tilde{\mu}$. Unfortunately, we show that in such a setting one cannot guarantee less than $\Omega(\sqrt{\eta})$ errors and rejections, but we show that $\wwtp$ nearly achieves such guarantees.

\paragraph{Experiments}
As a proof of concept, we perform simple controlled experiments on the task of handwritten letter classification using lower-case English letters from the EMNIST dataset (\cite{cohen2017emnist}). In one setup, to mimic a spamming adversary, after a classifier $h$ is trained, test examples are identified on which $h$ errs and are repeated many times in the test set. Existing SC algorithms (no matter how robust) will fail on such an example since they all choose $S$ without using unlabeled test examples---as long as an adversary can find even a single erroneous example, it can simply repeat it. In the second setup, we consider a natural test distribution which consists of a mix of lower- and upper-case letters, while the training set was only lower-case letters.
The simplest version of $\wtp$ achieves high accuracy while rejecting mostly adversarial or capital letters. 

\begin{figure}[t]
    \centering
    \includegraphics[width=\textwidth]{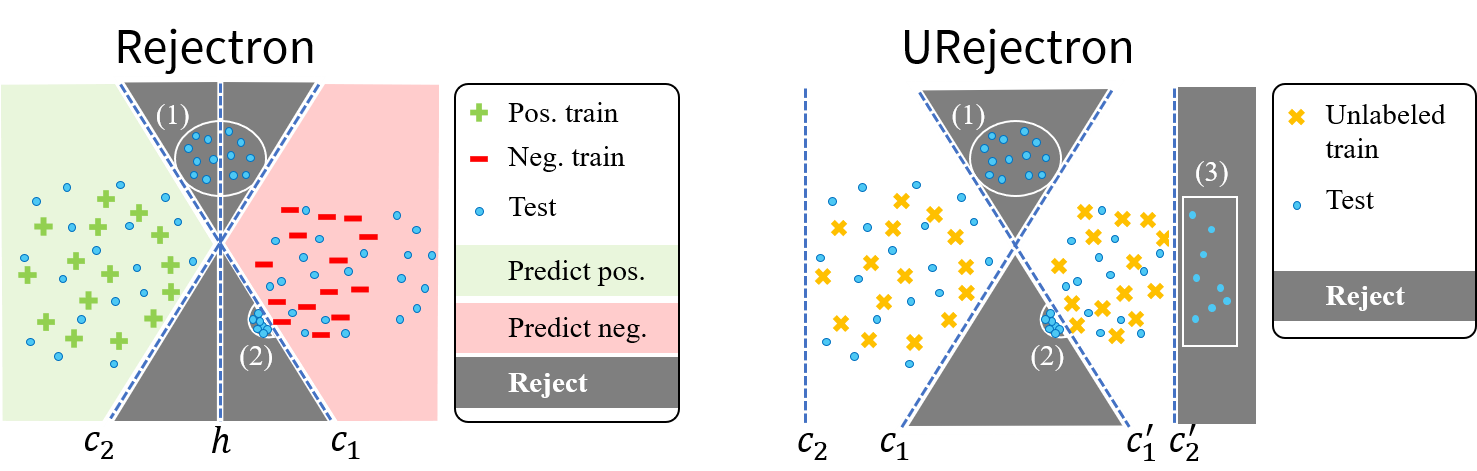}
    \caption{Our algorithm (and unsupervised variant)  for learning $C$=halfspaces.
    $\wwtp$ (left) first trains $h$ on labeled training data, then finds other candidate classifiers $c_1, c_2$, such that $h$ and $c_i$ have high disagreement on $\tbx$ and low disagreement on $\bx$, and rejects examples where $h$ and $c_i$ disagree.
     $\wtp$ (right) aims to distinguish \emph{unlabeled train} and test examples using pairs of classifiers $c_i, c_i'$ that agree on training data but disagree on many tests.
    Both reject:
    (1) clearly unpredictable examples which are very far from train and (2) a suspiciously dense cluster of tests which might all be positive despite being close to negatives.  $\wtp$ also rejects (3).
    }
    \label{fig:example}
\end{figure}
\paragraph{Organization} We next review related work in \Cref{sec:related}.  We present the learning setup in \Cref{sec:pqpac}. Our algorithm and guarantees are summarized in \Cref{sec:main}, followed by experiments (\Cref{sec:experiments}). Further discussion and future work are deferred to \Cref{sec:conclude}.

\section{Related work}\label{sec:related}

The redaction model combines SC and transductive learning, which have each been extensively studied, separately. We first discuss prior work on these topics, which (with the notable exception of online SC) has generally been considered when test examples are from the same distribution as training examples.

\paragraph{Selective classification} Selective classification go by various names including ``classification with a reject option'' and ``reliable learning.'' To the best of our knowledge, prior work has not considered SC using unlabeled samples from $Q\neq P$. Early learning theory work by \cite{RS88b} required a guarantee of 0 test errors and few rejections. However, \cite{kivinen1990reliable} showed that, for this definition, even learning rectangles under uniform distributions $P=Q$ requires exponential number of examples (as cited by \cite{hopkins2019power} which like much other work therefore makes further assumptions on $P$ and $Q$).
Most of this work assumes the same training and test distributions, without adversarial modification. \cite{kanade2009reliable} give a SC reduction to an agnostic learner (similar in spirit to our reduction to $\ERM$) but again for the case of $P=Q$. 

A notable exception is the work in \textit{online} SC, where an \textit{arbitrary sequence} of examples is presented one-by-one with immediate error feedback. This work includes the ``knows-what-it-knows'' algorithm \citep{li2011knows}, and  \cite{sayedi2010trading} exhibit an interesting trade-off between the number of mistakes and the number of rejections in such settings. However, basic classes such as intervals on the line are impossible to learn in these harsh online formulations. Interestingly, our division into labeled train and unlabeled test seems to make the problem easier than in the harsh online model.

\paragraph{Transductive (and semi-supervised) learning.} In transductive learning, the classifier is given test examples to classify all at once or in batches, rather than individually \citep[e.g.,][]{vapnik1998statistical}. Performance is measured with respect to the test examples. It is related to \textit{semi-supervised learning}, where unlabeled examples are given but performance is measured with respect to future examples from the same distribution. Here, since the assumption is that training and test examples are iid, it is generally the case that the unlabeled examples greatly outnumber the training examples, since otherwise they would provide limited additional value.

We now discuss related work which considers $Q \neq P$, but where classifiers must predict everywhere without the possibility of outputting $\question$.

\paragraph{Robustness to Adversarial Examples} There is ongoing effort to devise methods for learning predictors that are robust to adversarial examples \citep{szegedy2013intriguing,biggio2013evasion,goodfellowICLR15} at test time. Such work typically assumes that the adversarial examples are perturbations of honest examples chosen from~$P$.  
The main objective is to learn a classifier that has high robust accuracy, meaning that with high probability, the classifier will answer correctly even if the test point was an adversarially perturbed example. Empirical work has mainly focused on training deep learning based classifiers to be more robust \citep[e.g.,][]{DBLP:conf/iclr/MadryMSTV18, wong2018provable, zhang2019theoretically}. \cite{kang2019testing} consider the
fact that perturbations may not be known in advance, and some work \citep[e.g.,][]{pang2018towards} addresses the problem of identifying adversarial examples.  We emphasize that as opposed to this line of work, we consider {\em arbitrary} test examples and use SC.


Detecting adversarial examples has been studied in practice, but \cite{carlini2017adversarial} study ten proposed heuristics and are able to bypass all of them. Our algorithms also require a sufficiently large set of unlabeled test examples. The use of unlabeled data for improving robustness has also been empirically explored recently \citep[e.g.,][]{carmon2019unlabeled,stanforth2019labels,zhai2019adversarially}.

In work on real-world adversarial images, \cite{Yuan2019StealthyPU} find adversaries using highly visible transformations rather than imperceptible perturbations. They categorize seven major types of such transformations and write:
\begin{quote}
``Compared with the adversarial examples studied by the ongoing adversarial learning, such adversarial explicit  content does  not  need  to  be  optimized  in  a  sense that  the  perturbation  introduced  to  an  image  remains  less perceivable  to  humans.... today’s cybercriminals likely still rely on a set of predetermined obfuscation  techniques... not gradient descent.'' 
\end{quote}

\paragraph{Covariate Shift} The literature on learning with covariate shift is too large to survey here, see, e.g., the book by \cite{quionero2009dataset} and the references therein. To achieve guarantees, it is often assumed that the support of $Q$ is contained in the support of $P$. Like our work, many of these approaches use unlabeled data from $Q$ \citep[e.g.,][]{huang2007correcting, hardCS}. \cite{hardCS} show that  learning with covariate-shift is intractable, in the worst case, without such assumptions. In this work we overcome this negative result, and obtain guarantees for arbitrary~$Q$, using SC. 
In summary, prior work on covariate shift that guarantees low test/target error requires strong assumptions regarding the distributions. This motivates our model of covariate shift with rejections.



\section{Preliminaries and notation}
Henceforth, we assume a fixed class $C$ of $c:X\rightarrow Y$ from domain $X$ to $Y=\{0,1\}$,\footnote{For simplicity, the theoretical model is defined for binary classification, though our experiments illustrate a multi-class application. To avoid measure-theoretic issues, we assume $X$ is countably infinite or finite.} and let $d$ be the VC dimension of $C$.
Let $\log(x)=\log_2(x)$ denote the base-2 logarithm and $\ln(x)$ the natural logarithm. The set of functions from $X$ to $Y$ is denoted by $Y^X$. Let the set of subsets of $X$ be denoted by $2^X$. Finally, $[n]$ denotes $\{1,2,\ldots,n\}$ for any natural number $n \in \nats$.

\section{Learning with redaction}\label{sec:pqpac}
We now describe the two settings for SC. We use the same algorithm in both settings, so it can be viewed as two justifications for the same algorithm. The PQ model provides guarantees with respect to future examples from the test distribution, while the transductive model provides guarantees with respect to arbitrary test examples chosen by an all-powerful adversary. Interestingly, the transductive analysis is somewhat simpler and is used in the PQ analysis. 

\subsection{PQ learning}
In the \emph{PQ} setting, an SC learner $h|_S = L(\bx, f(\bx), \tbx)$ is given $n$ labeled examples $\bx = (x_1, \ldots, x_n)$ drawn iid $\bx \sim P^n$, labels $f(\bx)=(f(x_1),\ldots, f(x_n))$ for some unknown $f \in C$, and $n$ unlabeled examples $\tbx \sim Q^n$. $L$ outputs $h: X \rightarrow Y$ and $S \subseteq X$. The adversary (or nature) chooses $Q$ based only on $f, P$ and knowledge of the learning algorithm $L$.
The definition of PQ learning is given in \Cref{def:pq}. 
Performance is measured in terms of $\err_Q$ on future examples from $Q$ and $\rej_P$ (rather than the more obvious $\rej_Q)$. Rejection rates on $P$ (and $Q$) can be estimated from held out data, if so desired. The quantities $\rej_P, \rej_Q$ can be related and a small $\rej_P$ implies few rejections on future examples from $Q$ wherever it ``overlaps'' with $P$ by which we mean $Q(x) \leq \Lambda \cdot P(x)$ for some constant $\Lambda$.
\begin{lemma}\label{lem:convert}
For any $S\subseteq X$ and distributions $P,Q$ over $X$:
\begin{equation}\label{eq:tv_is_good}
\rej_Q(S) \leq \rej_P(S) + |P-Q|_\TV.
\end{equation}
Further, for any $\Lambda \geq 0,$
\begin{equation}\label{eq:where_errors_restated}
\Pr_{x \sim Q}\bigl[x \not\in S ~\text{ and }~ Q(x) \leq \Lambda P(x)\bigr] \leq \Lambda \rej_P(S).\end{equation}
\end{lemma}
\begin{proof}
For \cref{eq:tv_is_good}, note that one can sample a point from $\tilde{x}\sim Q$ by first sampling $x \sim P$ and then changing it with probability $|P-Q|_\TV$. This follows from the definition of total variation distance. Thus, the probability that $\tilde{x}$ is rejected is at most the probability $x$ is rejected plus the probability $x \neq \tilde{x}$, establishing \cref{eq:tv_is_good}. To see \cref{eq:where_errors_restated}, note
\[
    \Pr_{x \sim Q}\bigl[x \not\in S ~\text{ and }~ Q(x) \leq \Lambda P(x)\bigr] = \sum_{x \in \barS: Q(x) \leq \Lambda P(x)} Q(x) \leq \sum_{x \in \barS: Q(x) \leq \Lambda P(x)} \Lambda P(x).
\]
Clearly the above is at most
$\sum_{x \in \barS} \Lambda P(x) = \Lambda \rej_P$.
\end{proof}
If $\rej_P=0$ then all $x \sim Q$ that lie in $P$'s support would necessarily be classified (i.e., $x\in S$). Note that the bound \cref{eq:tv_is_good} can be quite loose and a tight bound is given in \Cref{ap:tight}.

It is also worth mentioning that a PQ-learner can also be used to guarantee  $\err_P +\rej_P\leq \eps$ meaning that it has \textit{accuracy} $\Pr_P[h|_S(x)=f(x)]\geq 1-\eps$ with respect to $P$ (like a normal PAC learner) but is also simultaneously robust to $Q$. The following claim shows this and an additional property that PQ learners can be made robust with respect to any polynomial number of different $Q$'s.
\begin{claim}\label{claim:further_robustness}
Let $f \in C, \eps, \delta>0,$ $n, k \geq 1$ and $P, Q_1, \ldots, Q_k$ be distributions over $X$. Given a $\bigl(\frac{\eps}{k+1}, \delta, n\bigr)$-PQ-learner $L$, $\bx \sim P^n$, $f(\bx)$, and additional unlabeled samples $\bz \sim P^n, \tbx_1\sim Q_1^n, \ldots,\tbx_k\sim Q_k^n$, one can generate $\tbx \in X^n$ such that $h|_S=L(\bx, f( \bx), \tbx)$ satisfies,
\[
\Pr\left[\rej_P +\err_P + \sum_i \err_{Q_i} \leq \eps\right]\geq 1-\delta.\]
\end{claim}
\begin{proof}[Proof of \Cref{claim:further_robustness}]
Let $Q=\frac{1}{k+1}\left(P+Q_1+\cdots +Q_k\right)$ be the blended distribution. Give $n$ samples from $P$ and each $Q_i,$ one can straightforwardly construct $n$ iid samples $\tbx\sim Q$. Running $L(\bx, f( \bx), \tbx)$ gives the guarantee that with prob.~$\geq 1-\delta$,
$(k+1)(\rej_P + \err_Q) \leq \eps$ which implies the claim since $(k+1)\err_Q=\err_P + \sum \err_{Q_i}$.
\end{proof}

\subsection{Transductive setting with white-box adversary}

In the \emph{transductive} setting, there is no $Q$ and instead empirical analogs $\err_\bx$ and $\rej_\bx$ of error and rejection rates are defined as follows, for arbitrary $\bx \in X^n$:
\begin{align}
\err_{\bx}(h|_S, f)&\defeq \frac{1}{n} |\{i\in[n]:  f(x_i)\neq h(x_i) ~\text{ and }~  x_i \in S\}| \label{eq:err}\\
\rej_{\bx}(S) &\defeq
\frac{1}{n}\bigl|\{i\in[n]:  x_i\notin S\}\bigr|
\end{align}
Again, $h, f$ and $S$ may be omitted when clear from context.

In this setting, the learner first chooses $h$ using only $\bx \sim P^n$ and $f(\bx)$. Then, a \emph{true} test set $\bz \sim P^n$ is drawn. Based on all available information ($\bx, \bz, f,h,$ and the code for learner $L$) the adversary modifies any number of examples from $\bz$ to create \emph{arbitrary} test set $\tbx \in X^n$. Finally, the learner chooses $S$ based on $\bx, f(\bx)$, and $\tbx$. Performance is measured in terms of $\err_\tbx + \rej_\bz$ rather than $\err_Q+\rej_P$, because $\bz \sim P^n$. One can bound $\rej_\tbx$ in terms of $\rej_\bz$ for any $\bz, \tbx \in X^n$ and $S \subseteq X$, as follows:
\begin{equation}\label{eq:ham}
    \rej_{\tbx} \leq \rej_\bz + \Delta(\bz, \tbx), ~~~\text{where}~~~ \Delta(\bz,\tbx)\defeq \frac{1}{n}\bigl|\{i\in[n]:  z_i\neq \tx_i\}\bigr|.
\end{equation}
The hamming distance $\Delta(\bz, \tbx)$ is the transductive analog of $|P-Q|_\TV$. The following bounds the ``false rejections,'' those unmodified examples that are rejected:
\begin{equation}\label{eq:spam}\frac{1}{n}\bigl|\{i \in [n]: \tx_i \not\in S ~\text{ and }~ \tx_i = z_i \}\bigr| \leq \rej_\bz(S).
\end{equation}
Both \cref{eq:ham,eq:spam} follow by definition of $\rej_{(\cdot)}$.

\paragraph{White-box adversaries}
The all-powerful transductive adversary is sometimes called ``white box'' in the sense that it can choose its examples while looking ``inside'' $h$ rather than using $h$ as a black box.  While it cannot choose $\tbx$ with knowledge of $S$, it can know what $S$ will be as a function of $\tbx$ if the learner is deterministic, as our algorithms are. Also, we note that the generalization analysis may be extended to a white-box model where the adversary chooses $Q$ knowing $h$, but it is cumbersome even to denote probabilities over $\tbx \sim Q^n$ when $Q$ itself can depend on $\bx \sim P^n$.

\section{Algorithms and guarantees}\label{sec:main}
We assume that we have a deterministic oracle  $\ERM=\ERM_C$ which, given a set of labeled examples from $X \times Y$, outputs a classifier $c \in C$ of minimal error. Figure~\ref{fig:algs} describes our algorithm $\wwtp$. It takes as input a set of labeled training examples $(\bx,\by)$, where $\bx\in X^n$ and $\by\in Y^n$, and a set of test examples $\tbx\in X^n$ along with an error parameter $\epsilon>0$ that trades off errors and rejections. A value for $\epsilon$ that theoretically balances these is in  \Cref{wow:real-gen,wow:real-trans}.

\begin{figure}[t!]
\rule[1ex]{\textwidth}{0.1pt}

$\wwtp(\text{train } \bx\in X^n, \text{labels } \by \in Y^n, \text{ test } \tbx \in X^n, \text{ error } \eps\in [0,1], \text{ weight } \Lambda =n+1):$
\begin{itemize}
\itemsep0em
    \item $h\coloneqq \ERM(\bx, \by)$  \text{ ~~\# assume black box oracle $\ERM$ to minimize errors}
    \item For $t=1,2,3,\ldots:$
    \begin{enumerate}
    \item $S_t \coloneqq \{x \in X: h(x)=c_1(x)=\ldots=c_{t-1}(x)\}$  ~~~~\# So $S_1=X$
    \item Choose $c_t \in \functions$ to maximize $s_t(c) \defeq \err_{\tbx}(h|_{S_t}, c) - \Lambda \cdot \err_\bx(h, c)$ over $c \in \functions$\vspace{0.04in}\\
    ~~~~~~\# \Cref{lem:eff} shows how to maximize $s_t$ using $\ERM$ ($\err$ is defined in \cref{eq:err})
    \item If $s_t(c_t) \leq \eps$, then stop and return $h|_{S_t}$
\end{enumerate}
\end{itemize}

\rule[1ex]{\textwidth}{0.1pt}
\caption{The $\wwtp$ algorithm takes labeled training examples and unlabeled test examples as input, and it outputs a selective classifier $h|_S$ that predicts $h(x)$ for $x \in S$ (and rejects all $x \not\in S$). Parameter $\eps$ controls the trade-off between errors and rejections and can be set to $\eps=\tilde\Theta(\sqrt{d/n})$ to balance the two.
The weight $\Lambda$ parameter is set to its default value of $n+1$ for realizable (noiseless) learning but should be lower for agnostic learning.\label{fig:algs}
}
\end{figure}

\begin{lemma}[Computational efficiency]\label{lem:eff}
For any $\bx, \tbx \in X^n$, $\by\in Y^n$, $\epsilon>0$ and $\Lambda\in \nats$, $\wwtp(\bx,\by,\tbx,\eps,\Lambda)$ outputs ${S_{T+1}}$ for $T\leq \lfloor 1/\eps\rfloor$. Further, each iteration can be implemented using one call to $\ERM$ on at
most $(\Lambda +1)n$ examples and $O(n)$ evaluations of classifiers in $\functions$.
\end{lemma}
\begin{proof}
To maximize $s_t$ using the ERM oracle for $C$, construct a dataset consisting of each training example, labeled by $h$, repeated $\Lambda$ times, and each test example in $\tilde{x}_i \in S_t$, labeled $1-h(\tilde{x}_i)$, included just once. Running $\ERM$ on this artificial dataset returns a classifier of minimal error on it. But the number of errors of classifier $c$ on this artificial dataset is:
\begin{align*}
&\Lambda \sum_{i \in [n]} |c(x_i)-h(x_i)| + \sum_{i: \tilde{x}_i \in S_t} |c(\tilde{x}_i)-(1-h(\tilde{x}_i))| =\\ &\Lambda \sum_{i \in [n]} |c(x_i)-h(x_i)| + \sum_{i: \tilde{x}_i \in S_t} 1-|c(\tilde{x}_i)-h(\tilde{x}_i)|,\end{align*}
which is equal to $\bigl|\{i\in [n]: \tilde{x}_i \in S_t\} \bigr|-n s_t(c)$. Hence $c$ minimizes error on this artificial dataset if and only if it maximizes $s_t$ of the algorithm.

Next, let $T$ be the number of iterations of the algorithm $\wwtp$, so its output is $h|_{S_{T+1}}$.
We must show that $T\leq \lfloor 1/\eps\rfloor$.  To this end, note that by definition, for every $t\in[T]$ it holds that $S_{t+1}\subseteq S_t$, and moreover,
\begin{equation}\label{eqn:shrink}
 \frac{1}{n}\bigl|\{i \in [n] : \tilde{x}_i \in S_t\}\bigr| - \frac{1}{n}\bigl|\{i \in [n] : \tilde{x}_i \in S_{t+1}\}\bigr|= \err_{\tbx}(h|_{S_t}, c_t) \geq s_t(c_t)> \epsilon.
\end{equation}
Hence, the fraction of additional rejected test examples in each iteration $t\in[T]$ is greater than $\epsilon$, and hence $T < 1/\eps$. Since $T$ is an integer, this means that $T \leq \lfloor 1/\eps\rfloor$.


For efficiency, of course each $S_t$ is not explicitly stored since even $S_1=X$ could be infinite. Instead, note that to execute the algorithm, we only need to maintain: (a) the subset of indices $Z_t= \{j \in [n]~|~\tilde{x}_j \in S_t\}$ of test examples which are in the prediction set, and (b) the classifiers $h, c_1, \ldots, c_T$. Also note that updating $Z_t$ from $Z_{t-1}$ requires evaluating $c_t$ at most $n$ times. In this fashion, membership in $S_t$ and $S={S_{T+1}}$ can be computed efficiently and output in a succinct manner.
\end{proof}
Note that since we assume $\ERM$ is deterministic, the $\wwtp$ algorithm is also deterministic. This efficient reduction to $\ERM$, together with the following imply that $\wwtp$ is a PQ learner:


\begin{theorem}[PQ guarantees]\label{wow:real-gen}
For any $n \in \nats, \delta>0, f \in C$ and distributions $P,Q$ over $X$:
\[\Pr_{\bx \sim P^n, \tbx\sim Q^n}[\err_Q\leq 2\eps^* ~\wedge~\rej_P  \leq \eps^*]\geq 1-\delta,\]
where $\eps^* = \sqrt{\frac{8d \ln 2n}{n}}+\frac{8 \ln 16/\delta}{n}$ and $h|_S=\wwtp(\bx,f(\bx),\tbx,\eps^*)$.
\end{theorem}
More generally, Theorem~\ref{thm:real-gen} shows that, by varying parameter $\eps$, one can achieve any trade-off between $\err_Q \leq O(\eps)$ and $\rej_P \leq \tilde{O}(\frac{d}{n\eps})$. The analogous transductive guarantee is:
\begin{theorem}[Transductive]\label{wow:real-trans}
For any $n \in \nats, \delta>0, f \in C$ and dist.~$P$ over $X$:
\[\Pr_{\bx,\bz\sim P^n}\left[\forall \tbx\in X^n:~\err_{\tbx}(h|_S)\leq \eps^*~\wedge~\rej_{\bz}(S)\leq \eps^*\right]\geq 1-\delta,\]
where $\eps^* = \sqrt{\frac{2d}{n}\log 2n}+\frac{1}{n}\log \frac{1}{\delta}$ and $h|_S=\wwtp(\bx,f(\bx),\tbx,\eps^*)$.
\end{theorem}
One thinks of $\bz$ as the real test examples and $\tbx$ as an arbitrary adversarial modification, not necessarily iid. \Cref{eq:spam} means that this implies $\leq \eps^*$ errors on unmodified examples. As discussed earlier, the guarantee above holds for {\em any} $\tbx$ chosen by a white-box adversary, which may depend on $\bx$ and $f$, and thus on $h$ (since $h=\ERM(\bx,f(\bx))$ is determined by $\bx$ and $f$).
More generally, \Cref{thm:real-trans} shows that, by varying parameter $\eps$, one can trade-off  $\err_{\tbx}\leq \eps$ and $\rej_\bz \leq \tilde{O}(\frac{d}{n\eps}).$

We note that \Cref{wow:real-gen,wow:real-trans} generalize in a rather straightforward manner to the case in which an adversary can inject additional training examples to form $\bx' \supseteq \bx$ which contains $\bx$.
Such an augmentation reduces the version space, i.e., the set of $h \in C$ consistent with $f$ on $\bx'$, but of course $f$ still remains in this set. The analysis remains essentially unchanged as long as $\bx'$ contains $\bx$ and $\bx$ consists of $n$ examples. The bounds remain the same in terms of $n$, but $\wwtp$ should be run with $\Lambda$ larger than the number of examples in $\bx'$ in this case to ensure that each $c_t$ has zero training error.

Here we give the proof sketch of \Cref{wow:real-trans}, since it is slightly simpler than  \Cref{wow:real-gen}. Full proofs are in \Cref{ap:otherproofs}.
\begin{proof}[Proof sketch for \Cref{wow:real-trans}]
To show $\err_\tbx \leq \eps^*$, fix any $f, \bx,\tbx$.   Since $h=\ERM(\bx, f(\bx))$ and $f \in C$, this implies that $h$ has zero training error, i.e., $\err_\bx(h, f)=0$. Hence $s_t(f)= \err_{\tbx}(h|_{S_t}, f)$ and the algorithm cannot terminate with $\err_{\tbx}(h|_{S_t}, f)  > \eps$ since it could have selected $c_t=f$.

To prove $\rej_\bz \leq \eps^*$, observe that $\wwtp$ never rejects any training $\bx$. This follows from the fact that $\Lambda > n$, together with the fact that $h(x_i)=f(x_i)$ for every $i\in[n]$ which follows, in turn, from the facts that $f\in C$ and $h=\ERM(\bx, f(\bx))$. Now $\bx$ and $\bz$ are identically distributed. By a generalization-like bound (\Cref{lem:lol}), with probability $\geq 1-\delta$ there is no classifier for which selects all of $\bx$ and yet rejects with probability greater than $\eps^*$ on $\bz$ for $T \leq 1/\eps^*$ (by \Cref{lem:eff}).
\end{proof}




Unfortunately, the above bounds are worse than standard $\tilde{O}(d/n)$ VC-bounds for $P=Q$, but the following lower-bound shows that $\tilde{O}(\sqrt{d/n})$ is tight for some class $C$.
\begin{theorem}[PQ lower bound]\label{thm:lower-gen}
There exists a constant $K>0$ such that:
for any $d \geq 1$, 
there is a concept class $C$ of VC dimension $d$,  distributions $P$ and $Q$, such that for any $n\geq 2d$ and learner $L:X^n\times Y^n \times X^n \rightarrow Y^X \times 2^X$, there exists $f \in C$ with
\[\E\nolimits_{\substack{\bx \sim P^n\\ \tbx \sim Q^n}} \left[ \rej_{P} + \err_{Q} \right] \geq K \sqrt{\frac{d}{n}}, ~~~\text{where}~~~h|_S = L(\bx, f(\bx), \tbx).\]
\end{theorem}
Note that since $P$ and $Q$ are fixed, independent of the learner $L$, the unlabeled test examples from $Q$ are not useful for the learner as they could simulate as many samples from $Q$ as they would like on their own. Thus, the lower bound holds even given $n$ training examples and $m$ unlabeled test examples, for arbitrarily large $m$. 

\Cref{thm:lower-gen} implies that the learner needs at least $n=\Omega(d/\eps^2)$ labeled training examples to get the $\epsilon$ error plus rejection guarantee.  However, it leaves open the possibility that many fewer than $m = \tilde{O}(d/\eps^2)$ test examples are needed. We give a lower bound in the transductive case which shows that both $m,n$ must be at least ${\Omega}(d/\eps^2)$:
\begin{theorem}[Transductive lower bound]\label{thm:lower-trans}
There exists a constant $K > 0$ such that: for any $d\geq 1$ there exists a concept class of VC dimension $d$ where, for any $m, n \geq 4d$ there exists a distribution $P$, and an adversary $\cA: X^{n+m} \rightarrow X^m$, such that
for any learner $L:X^n \times Y^n \times X^m \rightarrow Y^X \times 2^X$ there is a function $f \in C$ such that:
\[\E\nolimits_{\substack{\bx \sim P^n\\ \bz \sim P^m}}[\rej_\bz + \err_\tbx] \geq K\sqrt{\frac{d}{\min\{m,n\}}}\]
where $\tbx = \cA(\bx, \bz)$ and $h|_S=L(\bx, f(\bx), \tbx)$.
\end{theorem}
This $\Omega(\sqrt{d/\min\{m,n\}})$ lower bound implies that one needs both $\Omega(d/\eps^2)$ training and test examples to guarantee $\eps$ error plus rejections. This is partly why, for simplicity, aside from the \Cref{thm:lower-trans}, our analysis takes $m=n$. The proofs of these two lower bounds are in \Cref{ap:lower-proofs}.

\begin{figure}[t!]\label{fig:WTP}
\rule[1ex]{\textwidth}{0.1pt}

$\wtp(\text{train } \bx\in X^n, \text{test } \tbx \in X^n, \text{error } \eps\in [0,1], \text{weight } \Lambda =n+1):$

\begin{itemize}
\itemsep0em
    \item For $t=1,2,3,\ldots:$
    \begin{enumerate}
    \item $S_t \coloneqq \{x \in X: c_1(x)=c_1'(x) \wedge \cdots \wedge c_{t-1}(x)=c_{t-1}'(x)\}$  ~~\# So $S_1 = X$
    \item Choose $c_t, c'_t \in \functions$ to maximize $s_t(c,c') \defeq \err_{\tbx}(c'|_{S_t}, c) - \Lambda \cdot \err_\bx(c', c)$\vspace{0.04in}\\
    ~~~~~~\# \Cref{lem:eff-wtp} shows how to maximize $s_t$ using $\ERM_{\DIS}$ ($\DIS$ is defined in \cref{eq:DIS})
    \item If $s_t(c_t, c'_t) \leq \eps$, then stop and return $S_t$
\end{enumerate}
\end{itemize}

\rule[1ex]{\textwidth}{0.1pt}
\caption{The $\wtp$ unsupervised algorithm takes unlabeled training examples and unlabeled test examples as input, and it outputs a set $S\subseteq X$ where classification should take place.
}\label{fig:nalgs}
\end{figure}

\paragraph{Unsupervised selection algorithm.}
Our unsupervised selection algorithm $\wtp$ is described in Figure~\ref{fig:nalgs}.  It takes as input only train and test examples $\bx,\tbx\in X^n$ along with an error parameter $\epsilon$ recommended to be $\tilde\Theta(\sqrt{d/n})$, and it outputs a set~$S$ of the selected elements.
$\wtp$ requires a more powerful black-box ERM---we show that $\wtp$ can be implemented efficiently if one can perform ERM with respect to the family of binary classifiers that are disagreements (xors) between two classifiers. For classifiers $c,c':X\rightarrow Y$, define $\dis_{c, c'}:X \rightarrow \{0,1\}$ and $\DIS$ as follows:
\begin{equation}
\dis_{c,c'}(x)\defeq\begin{cases}1 & \text{if } c(x)\neq c'(x)\\0 & \text{otherwise}\end{cases} ~~~\text{ and }~~~
\DIS\defeq\{\dis_{c,c'}:~c,c'\in C\}.\label{eq:DIS}
\end{equation}
\Cref{lem:eff-wtp} shows how $\wtp$ is implemented efficiently with an $\ERM_\DIS$ oracle.

Also, we show nearly identical guarantees to those of \Cref{wow:real-trans} for $\wtp$:
\begin{theorem}[Unsupervised]\label{cor:wtp-real-trans}
For any $n \in \nats$, any $\delta\geq 0$, and any distribution $P$ over $X$:
\[\Pr_{\bx,\bz\sim P^n}\left[\forall f\in C, \tbx \in X^n:~\left(\err_{\tbx}(h|_S) \leq \eps^*\right)\wedge\left(\rej_{\bz}(S)\leq \eps^*  \right)\right]\geq 1-\delta,\]
where $\eps^*=\sqrt{\frac{2d}{n}\log 2n} + \frac{1}{n}\log \frac{1}{\delta}$, $S=\wtp(\bx, \tbx, \eps^*)$ and $h=\ERM_C(\bx, f(\bx)).$
\end{theorem}
The proof is given in \Cref{sec:urejectron-analysis} and follows from  
Theorem~\ref{thm:wtp-real-trans} which shows that by varying parameter $\epsilon$, one can achieve any trade-off $\err_{\tbx}\leq \epsilon$ and $\rej_{\bz}\leq \tilde{O}(\frac{d}{n\eps})$. Since one runs $\wtp$ without labels, it has guarantees with respect to any empirical risk minimizer $h$ which may be chosen separately, and its output is also suitable for a multi-class problem.

\paragraph{Massart noise.} We also consider two non-realizable models. First, we consider the Massart noise model, where there is an arbitrary (possibly adversarial) noise rate $\eta(x) \leq \eta$ chosen for each example.  We show that $\wwtp$ is a PQ learner in the Massart noise model with $\eta< 1/2$, assuming an ERM oracle and an additional $N=\tilde{O}\left(\frac{d n^2 }{\delta^2(1-2\eta)^2}\right)$  examples from $P$.  See \Cref{ap:massart} for details.

\paragraph{A semi-agnostic setting.} 
We also consider the following semi-agnostic model. 
For an arbitrary distribution $D$ over $X \times Y$, again with $Y=\{0,1\}$, the analogous notions of rejection and error are:
\[\rej_D(S) \defeq \Pr_{(x,y)\sim D}[x \not\in S]  ~~~\text{and}~~~
\err_D(h|_S) \defeq \Pr_{(x,y)\sim D}[h(x)\neq y \wedge x \in S]
\]
In standard agnostic learning with respect to $D$, we suppose there is some classifier $f \in C$ with error $\err_D(f)\leq \eta$ and we aim to find a classifier whose generalization error is not much greater than $\eta$. In that setting, one can of course choose $\eta_{\text{opt}} \defeq \min_{f \in C} \err_D(f)$. For well-fitting models, where there is some classifier with very low error, $\eta$ may be small.

To prove any guarantees in our setting, the test distribution must somehow be related to the training distribution. To tie together the respective training and test distributions $\mu, \tmu$ over $X\times Y$, we suppose we know $\eta$ such that both $\err_\mu(f)\leq \eta$ and $\err_\tmu(f)\leq \eta$ for some $f \in C$. Even with these conditions,
\Cref{ag-lower} shows that one cannot simultaneously guarantee error rate on $\tmu$ and rejection rate on $\mu$ less than $\sqrt{\eta/8}$, and \Cref{wow:ag-gen} shows that our $\wwtp$ algorithm achieves a similar upper bound. This suggests that PQ-learning (i.e., adversarial SC) may be especially challenging in settings where ML is not able to achieve low error $\eta$.



\section{Experiments}\label{sec:experiments}
Rather than classifying sensitive attributes such as explicit images, we perform simple experiments on handwritten letter classification from the popular EMNIST dataset \citep{cohen2017emnist}. For both experiments, the training data consisted of the eight lowercase letters \textit{a d e h l n r t}, chosen because they each had more than 10,000 instances. From each letter, 3,000 instances of each letter were reserved for use later, leaving 7,000 examples, each constituting 56,000 samples from $P$.

We then considered two test distributions, $\qadv, \qnat$ representing adversarial and natural settings. $\qadv$ consisted of a mix of 50\% samples from $P$ (the 3,000 reserved instances per lower-case letter mentioned above) and 50\% samples from an adversary that used a classifier $h$ as a black box. To that, we added 3,000 adversarial examples for each letter selected as follows: the reserved 3,000 letters were labeled by $h$ and the adversary selected the first misclassified instance for each letter. Misclassified examples are shown in Figure \ref{fig:evil}. It made 3,000 imperceptible modifications of each of the above instances by changing the intensity value of a single pixel by at most 4 (out of 256). The result was 6,000 samples per letter constituting 48,000 samples from $\qadv$.

\begin{figure}
\centering
\begin{subfigure}{.264\textwidth}
  \includegraphics[width=\textwidth]{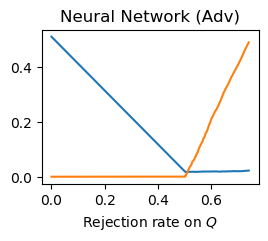}
\end{subfigure}%
\begin{subfigure}{.246\textwidth}
  \includegraphics[width=\textwidth]{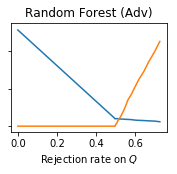}
\end{subfigure}%
\begin{subfigure}{.243\textwidth}
  \includegraphics[width=\textwidth]{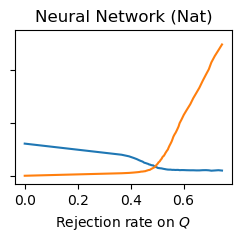}
\end{subfigure}%
\begin{subfigure}{.246\textwidth}
  \includegraphics[width=\textwidth]{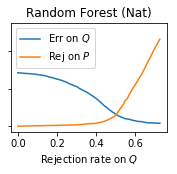}
\end{subfigure}
\caption{Trade-offs between rejection rate on $P$ and error rate on $Q$. 
The error on $Q$ (in blue) is the fraction of errors \emph{among selected examples} (unlike $\err_Q$ in our analysis). }
\label{fig:tradeoff}
\end{figure}

For $\qnat$, the test set also consisted of 6,000 samples per letter, with 3,000 reserved samples from $P$ as above. In this case, the remaining half of the letters were simply upper-case\footnote{In some datasets, letter classes consist of a mix of upper- and lower-case, while in others they are assigned different classes (EMNIST has both types of classes). In our experiments, they belong to the same class.} versions  of the letters \textit{A D E H L N R T}, taken from the EMNIST dataset (case information is also available in that dataset). Again the dataset size is 48,000. We denote this test distribution by $\qnat$.


In Figure~\ref{fig:tradeoff}, we plot the trade-off between the rejection rate on $P$ and the error rate on $Q$ of the $\wtp$ algorithm. Since this is a multi-class problem, we implement the most basic form of the $\wtp$ algorithm, with $T=1$ iterations. Instead of fixing parameter $\Lambda$, we simply train a predictor $h^{\rm Dis}$ to distinguish between examples from $P$ and $Q$, and train a classifier $h$ on $P$. We trained two models, a random forest (with default parameters from scikit-learn \citep{pedregosa2011scikit}) and a neural network. Complete details are provided at the end of this section. We threshold the prediction scores of distinguisher $h^{\rm Dis}$ at various values. For each threshold $\tau$, we compute the fraction of examples from $P$ that are rejected (those with prediction score less than $\tau$), and similarly for $Q$, and the error rate of classifier $h$ on examples from $Q$ that are \textit{not} rejected (those with prediction score at least $\tau$). We see in Figure~\ref{fig:tradeoff} that for a suitable threshold $\tau$ our $\wtp$ algorithm achieves both low rejection rate on $P$ and low error rate on $Q$. Thus on these problems the simple algorithm suffices.

We compare to the state-of-the-art SC algorithm SelectiveNet \citep{GeifmanE19}. We ran it to train a selective neural network classifier on $P$. SelectiveNet performs exceptionally on $\qnat$, achieving low error and reject almost exclusively upper-case letters. It of course errs on all adversarial examples from $\qadv$, as will all existing SC algorithms (no matter how robust) since they all choose $S$ without using unlabeled test examples.

\paragraph{Models} A Random Forest Classifier $h_{\rm RF}$ from Scikit-Learn (default parameters including 100 estimators) \citep{pedregosa2011scikit} and a simple neural network $h_{\rm NN}$ consisting of two convolutional layers followed by two fully connected layers\footnote{\url{https://github.com/pytorch/examples/blob/master/mnist/main.py}} were fit on a stratified 90\%/10\% train/test split. The network parameters are trained with SGD with momentum ($0.9$), weight decay ($2\times 10^{-4}$), batch size ($128$), for $85$ epochs with a learning rate of $0.1$, that is decayed it by a factor of 10 on epochs 57 and 72. $h_{\rm RF}$ had a 3.6\% test error rate on $P$, and $h_{\rm NN}$ had a 1.3\% test error rate on $P$. 

\paragraph{SelectiveNet}
SelectiveNet requires a target coverage hyperparameter which in our experiments is fixed to 0.7. We use an open-source PyTorch implementation of SelectiveNet that is available on GitHub \footnote{\url{https://github.com/pranaymodukuru/pytorch-SelectiveNet}}, with a VGG 16 architecure \citep{DBLP:journals/corr/SimonyanZ14a}. To accommodate the  VGG 16 architecure without changes, we pad all images with zeros (from 28x28 to 32x32), and duplicate the channels (from 1 to 3). SelectiveNet achieves  rejection rates of 21.08\% ($P$), 45.89\% ($\qnat$), and 24.04\% ($\qadv$), and error rates of 0.02\% ($P$), 0.81\% ($\qnat$), and 76.78\% ($\qadv$).

\begin{figure}
    \centering
    \includegraphics{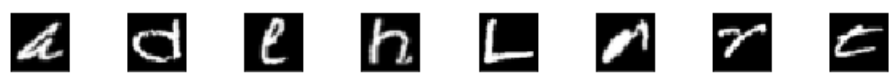}
    \caption{Adversarial choices of \textit{a d e h l n r t}, misclassified by the Random Forest classifier.}
    \label{fig:evil}
\end{figure}

\section{Conclusions}\label{sec:conclude}
The fundamental theorem of statistical learning states that an $\ERM$ algorithm for class $C$ is asymptotically nearly optimal requiring $\tilde{\Theta}(d/n)$ labeled examples for learning arbitrary distributions when $P=Q$ \citep[see, e.g.,][]{shalev2014understanding}. This paper can be viewed as a generalization of this theorem to the case where $P\neq Q$, obtaining $\tilde{\Theta}(\sqrt{d/n})$ rates. When $P=Q$, unlabeled samples from $Q$ are readily available by ignoring labels of some training data, but unlabeled test samples are necessary when $P\neq Q$. No prior such guarantee was known for arbitrary $P\neq Q$, even for simple classes such as intervals, perhaps because it may have seemed impossible to guarantee anything meaningful in the general case.

The practical implications are that, to address learning in the face of adversaries beyond perturbations (or drastic covariate shift), unlabeled examples and abstaining from classifying may be necessary. In this model, the learner can beat an unbounded white-box adversary. Even the simple approach of training a classifier to distinguish unlabeled train vs.\ test examples may be adequate in some applications, though for theoretical guarantees one requires somewhat more sophisticated algorithms.



\section*{Broader Impact}
In adversarial learning, this work can benefit users when adversarial examples are correctly identified. It can harm users by misidentifying such examples, and the misidentifications of examples as suspicious could have negative consequences just like misclassifications. This work ideally could benefit groups who are underrepresented in training data, by abstaining rather than performing harmful incorrect classification. However, it could also harm such groups: (a) by providing system designers an alternative to collecting fully representative data if possible; (b) by harmfully abstaining at different rates for different groups;  (c) when those labels would have otherwise been correct but are instead being withheld; and (d) by identifying them when they would prefer to remain anonymous.

Our experiments on handwriting recognition have few ethical concerns but also have less ecological validity than real-world experiments on classifying explicit images or medical scans.




\bibliographystyle{plainnat}
\bibliography{learning}
\newpage

\appendix




\section{Rejectron analysis (realizable)}\label{ap:otherproofs}

In this section, we present the analysis of $\wwtp$ in the realizable case $f \in C$. Say a classifier $c$ is \textit{consistent} if $c(\bx)=f(\bx)$ makes 0 training errors. \Cref{wow:real-trans} provides transductive guarantees on the empirical error and rejection rates, while \Cref{wow:real-gen} provides generalization guarantees that apply to future examples from $P, Q$. Both of these theorems exhibit trade-offs between error and rejection rates. At a high level, their analysis has the following structure:
\begin{itemize}
    \item $\wwtp$ selects a consistent $h=\ERM(\bx, f(\bx))$, since we are in the realizable case.
    \item Each $c_t$ is a consistent classifier that disagrees with $h|_{S_t}$ on the tests $\tbx$ as much as possible, with $s_t(c_t)=\err_\tbx(h|_{S_t}, c_t)$ (since  $\err_\bx(h, c_t)=0$). This follows the facts that $\Lambda > n$,  $s_t(h)=0$, and $s_t(c)<0$ for any inconsistent $c$. (The algorithm is defined for general $\Lambda<n$ for the agnostic analysis later.)
    \item Therefore, when the algorithm terminates on iteration $T$, it has empirical test error $\err_\tbx(h|_{S_T}, f)\leq \eps$ otherwise it could have chosen $c_t=f$.
    \item The number of iterations $T < 1/\eps$ since on each iteration an additional $\eps$ fraction of $\tbx$ is removed from $S_t$. Lemma \ref{lem:eff} states this and shows how to use an $\ERM$ oracle on an artificial dataset to efficiently find $c_t$.
    \item All training examples $x_i$ are in $S$ since each $c_t$ and $h$ agree on all $x_i$.
    \item Transductive error and rejection bounds:
    \begin{enumerate}
        \item For error, we have already argued that the empirical error $\err_\tbx \leq \eps$.
        \item For rejection, \Cref{lem:lol} states that it is unlikely that there would be any choice of $h, \bc=(c_1, \ldots, c_T)$ where the resulting $S(h, \bc)\defeq \{x \in X: h(x)=c_1(x)=\ldots=c_T(x)\}$ would contain all training examples but reject (abstain on) many ``true'' test examples $z_i$ since $\bx$ and $\bz$ are identically distributed. The proof uses Sauer's lemma.
    \end{enumerate}
    \item Generalization error and rejection bounds:
    \begin{enumerate}
        \item For error, \Cref{lem:omg} states that it is unlikely that there is any $h, \bc$ such that $\err_\tbx(h|_{S(h, \bc)})\leq \eps$ yet $\err_Q(h|_{S(h, \bc)})>2\eps$.
        \item For rejection rate, \Cref{lem:omg2} uses VC bounds to show that it is unlikely that $\rej_P(S(h, \bc))>\eps$ while $\rej_\bx(S(h, \bc))=0$.
    \end{enumerate}
    Both proofs use Sauer's lemma.
\end{itemize}

We next move to the transductive analysis since it is simpler, and it is also used as a stepping stone to the generalization analysis.

\subsection{Transductive guarantees (realizable)}

Note that $\wwtp$ rejects any $x\notin S$, where $S=S(h,\bc)$ is defined by
\begin{equation}\label{eq:S}
    S(h,\bc)\defeq \bigl\{x \in X: h(x)=c_1(x)=c_2(x)=\ldots=c_T(x)\bigr\}.
\end{equation}
In what follows, we prove the transductive analogue of a ``generalization'' guarantee for arbitrary $h \in C, \bc\in C^T$.  This will be useful when proving \Cref{wow:real-trans}.
\begin{lemma}\label{lem:lol}
For any $T,n\in\mathbb{N}$, any $\delta\geq 0$, and $\eps=\frac{1}{n}\left(d(T+1) \log(2n) + \log \frac{1}{\delta}\right)$:
\[\Pr_{\bx, \bz\sim P^n}\left[\exists \bc \in C^T, h \in C:~(\rej_\bx({S(h,\bc)})=0) \wedge (\rej_\bz({S(h,\bc)})>\eps) \right] \leq \delta.\]
\end{lemma}
This lemma is proven in \Cref{sec:gen}.
Using it, we can show a trade-off between error and rejection rate for the transductive case.
\begin{theorem}\label{thm:real-trans}
For any $n\in\mathbb{N}$, any $\eps,\delta\geq 0$, any $f \in C$:
\begin{equation}\label{eqn:realize-trans-err}
\forall \bx,\tbx\in X^n:~\err_{\tbx}\bigl(h|_S, f\bigr)  \leq \eps,
\end{equation}
where $h|_S=\wwtp(\bx,f(\bx),\tbx,\eps)$,
and for any distribution $P$ over $X$,
\begin{equation}\label{eqn:realize-trans-acc}
\Pr_{\bx,\bz\sim P^n}\left[\forall \tbx\in X^n:~\rej_{\bz}\bigl(S\bigr)
\leq \frac{1}{n}\left(\frac{2d}{\eps} \log(2n) + \log \frac{1}{\delta}\right) \right]\geq 1-\delta.
\end{equation}
\end{theorem}

We note that a natural alternative formalization of Equation~\eqref{eqn:realize-trans-acc} would be to require that
\[
\Pr_{\bx\sim P^n}\left[\forall \tbx\in X^n:~\rej_{P}\bigl(S\bigr)
\leq \frac{1}{n}\left(\frac{2d}{\eps} \log(2n) + \log \frac{1}{\delta}\right) \right]\geq 1-\delta.
\]
However, the formalization of Equation~\eqref{eqn:realize-trans-acc} is stronger, as it guarantees that the rejection probability is small, even if the adversary is {\em ``white-box"} and chooses $\tilde{\bf x}$ after seeing~${\bf z}$.

\begin{proof}[Proof  of \Cref{thm:real-trans}]
We start by proving \cref{eqn:realize-trans-err}.  To this end, fix any $n\in\mathbb{N}$, any $\epsilon>0$, any $f\in\functions$, and any $\bx,\tbx\in X^n$.  Let
$h=\ERM(\bx, f(\bx))$.  Since we are in the realizable case, this implies that $h$ has zero training error, i.e., $\err_\bx(h, f)=0$, and hence $s_t(h)=\err_{\tbx}(h|_{S_t}, f)$ for all $t$. Thus, the algorithm cannot terminate on any iteration where $\err_{\tbx}(h|_{S_t}, f)  > \eps$ since it can always select $c_t=f\in C$.  This proves Equation~\eqref{eqn:realize-trans-err}.

It remains to prove \cref{eqn:realize-trans-acc}.  By \Cref{lem:eff}, $T=\lfloor 1/\eps \rfloor$ is an upper bound on the number of completed iterations of the algorithm. WLOG there are exactly $T$ iterations because if there were actually $T'<T$ iterations, simply ``pad'' them with $c_{T'+1}=\ldots=c_T=h$ which doesn't change~$S$.

We note that the algorithm selects all training examples.  This follows from the fact that $\Lambda > n$, together with the fact that $h(x_i)=f(x_i)$ for every $i\in[n]$, where the latter follows from the fact that  $f\in C$ and $h=\ERM(\bx, f(\bx))$. By \Cref{lem:lol}, with probability $\geq 1-\delta$ there are no choices $h \in C, \bc=(c_1, \ldots, c_T) \in C^T$ for which $S(h, \bc)$ contains all $x_i$'s but is missing $\geq \eps'$ fraction of $\bz$ for $\eps'=\frac{1}{n}\left(\frac{2d}{\eps} \log(2n) + \log \frac{1}{\delta}\right)$ since $T+1\leq 2/\eps$.
\end{proof}

\Cref{wow:real-trans} is a trivial corollary of~ \Cref{thm:real-trans}.
\begin{proof}[Proof of \Cref{wow:real-trans}]
Recall $\eps^* = \sqrt{\frac{2d}{n}\log 2n}+\frac{1}{n}\log \frac{1}{\delta}$ and $h|_S=\wwtp(\bx,f(\bx),\tbx,\eps^*)$.
The proof follows from \Cref{thm:real-trans} and the fact that:
\[ \frac{1}{n}\left(\frac{2d}{\eps^*}\log 2n  + \log \frac{1}{\delta}\right)\leq \frac{2d \log 2n}{n\sqrt{\frac{2d}{n}\log 2n}}+\frac{1}{n}\log\frac{1}{\delta} =\eps^*.\]
\end{proof}

\subsection{Generalization guarantees (realizable)}
Before we state our generalization guarantees, analogous to \Cref{lem:lol} above, we prove that low test error and low training rejection rates imply, with high probability, low generalization error and rejection rates.
\begin{lemma}\label{lem:omg}
For any $\delta>0, \eps\geq \frac{8 \ln 8/\delta}{n} + \sqrt{\frac{8d \ln 2n}{n}}$, $T\leq 1/\eps$, any $f,h \in C$ and any distribution $Q$ over $X$,
\[\Pr_{\bz\sim Q^n}\left[\exists\bc \in C^T: \left(\err_{Q}(h|_{S(h,\bc)},f)>2\eps\right)\wedge\left(\err_\bz(h|_{S(h,\bc)},f)\leq \eps\right)\right]\leq \delta.
\]
\end{lemma}

\begin{lemma}\label{lem:omg2}
For any $T\geq 1$, any $f \in C$ and any distribution $P$ over $X$,
\[\Pr_{\bx\sim P^n}\left[\exists h\in C,\bc \in C^T: \left(\rej_P({S(h,\bc)})>\xi\right)\wedge\left(\rej_\bx(S(h,\bc))=0\right)\right]\leq \delta,
\]
where $\xi=\frac{2}{n}(d(T+1)\log(2n)+\log \frac{2}{\delta})$. Also,
\[\Pr_{\bx\sim P^n}\left[\exists h\in C,\bc \in C^T: \left(\rej_P(S(h,\bc))>2\alpha\right)\wedge\left(\rej_\bx(S(h,\bc))\leq \alpha\right)\right]\leq \delta,
\]
for any $\alpha\geq \frac{8}{n}(d(T+1)\ln(2n) + \ln \frac{8}{\delta})$.

\end{lemma}
We mention that the first inequality in \Cref{lem:omg2}  is used to provide generalization guarantees in the realizable setting, whereas the latter inequality is used to provide guarantees in the semi-agnostic setting.

\begin{theorem}\label{thm:real-gen}
For any $n \in\mathbb{N}$ and  $\delta>0$, any $\eps \geq  \sqrt{\frac{8d \ln 2n}{n}}+\frac{8 \ln 8/\delta}{n}$, any $f \in C$ and any distributions $P,Q$ over $X$:
\begin{equation}\label{eqn:realize-gen-err}
\forall \bx\in X^n:~~\Pr_{\btx \sim Q^n}\left[\err_Q(h|_S) \leq 2\eps\right] \geq 1-\delta,
\end{equation}
where $h|_S\defeq \wwtp(\bx,f(\bx),\btx,\eps)$. Furthermore, for any $\eps\geq 0$,
\begin{equation}\label{eqn:realize-gen-acc}
\Pr_{\bx \sim P^n}\left[\forall \btx\in X^n:~\rej_P  \leq \frac{2}{n}\left(\frac{2d}{\eps}\log 2n + \log \frac{2}{\delta}\right)\right] \geq 1-\delta.
\end{equation}
\end{theorem}
\begin{proof}[Proof of \Cref{thm:real-gen}]
Let $T= \lfloor 1/\eps\rfloor$ be an upper bound on the number of iterations.
We first prove \cref{eqn:realize-gen-err}. Since the ERM algorithm is assumed to be deterministic, the function $h$ is uniquely determined by $\bx$ and $f$. By Theorem~\ref{thm:real-trans} (Equation~\eqref{eqn:realize-trans-err}), the set $S$ has the property that $\err_\btx(h|_S)\leq \eps$ (with certainty) for all $\bx, \btx$. By \Cref{lem:omg}, with probability at most $\delta$ there exists a choice of $h, \bc$ which would lead to $\err_Q(h|_S)> 2\eps$ and $\err_\btx(h|_S)\leq \eps$, implying \cref{eqn:realize-gen-err}.

For \cref{eqn:realize-gen-acc}, as we argued in the proof of Theorem~\ref{thm:real-trans}, the fact that $\Lambda > n$, together with the fact we are in the realizable case (i.e., $\by=f(\bx)$), implies that we select all training examples. Because of this and the fact that $T+1\leq 2/\eps$, \Cref{lem:omg2} implies \cref{eqn:realize-gen-acc}.
\end{proof}
\Cref{wow:real-gen} is a trivial corollary of~\Cref{thm:real-gen}.
\begin{proof}[Proof of~\Cref{wow:real-gen}]
Recall that $\eps^* = \sqrt{\frac{8d \ln 2n}{n}}+\frac{8 \ln 16/\delta}{n}$.

\Cref{eqn:realize-gen-err} implies that  $\Pr[\err_Q\leq 2\eps^*]\geq 1-\delta/2$ and \cref{eqn:realize-gen-acc} implies,
\[\Pr_{\bx \sim P^n}\left[\forall \bz\in X^n:~\rej_P \leq \frac{2}{n}\left(\frac{2d}{\eps^*}\log 2n + \log \frac{4}{\delta}\right)\right] \geq 1- \frac{\delta}{2}.\]
Further, note that $\log_2 r \leq 2\ln r$ for $r \geq 1$ and hence, using $\eps^* >\sqrt{\frac{8d \ln 2n}{n}}$,
\[\frac{2}{n}\left(\frac{2d}{\eps^*}\log 2n + \log \frac{4}{\delta}\right) \leq \frac{8d}{n\eps^*}\ln 2n + \frac{4}{n}\ln \frac{4}{\delta} \\
< \sqrt{\frac{8d \ln 2n}{n}} + \frac{4}{n}\ln \frac{4}{\delta} \leq \eps^*.
\]
The proof is completed by the union bound.
\end{proof}

\section{Analysis of Urejectron}\label{sec:urejectron-analysis}

In this section we present a transductive analysis of $\wtp$, again in the realizable case. We begin with its computational efficiency.

\begin{lemma}[$\wtp$ computational efficiency]\label{lem:eff-wtp}
For any $\bx, \tbx \in X^n,\epsilon>0$ and $\Lambda\in \nats$, $\wtp$ outputs $S_{T+1}$ for $T\leq \lfloor 1/\eps\rfloor$. Further, each iteration can be implemented using one call to $\ERM_{\DIS}$, as defined in \cref{eq:DIS}, on at most $(\Lambda +1)n$ examples and $O(n)$ evaluations of classifiers in $\functions$.
\end{lemma}
The proof of this lemma is nearly identical to that of \Cref{lem:eff}.
\begin{proof}[Proof of \Cref{lem:eff-wtp}]
The argument that $T \leq \lfloor 1/\eps \rfloor$ follows for the same reason as before, replacing \cref{eqn:shrink} with:
\[|\{i:~x_i \in S_t\}|-|\{i:~x_i \in S_{t+1}\}|=|\{i: x_i \in S_t \wedge c_t(x_i)\neq c'_t(x_i)\}| = n \err_{\tbx}(c_t|_{S_t}, c'_t)\geq n\eps.\]
For efficiency, again all that needs to be stored are the subset of indices $Z_t= \{i~|~\tilde{x}_i \in S_t\}$ and the classifiers $c_1, c'_1, \ldots, c_T, c'_T$ necessary to compute $S$. To implement iteration $t$ using the $\ERM_\DIS$ oracle, construct a dataset consisting of each training example, labeled by 0, repeated $\Lambda$ times, and each test example in $\tilde{x}_i \in S_t$, labeled $1$, included just once. The accuracy of $\dis_{c,c'}$ on this dataset is easily seen to differ by a constant from $s_t(c,c')$, hence running $\ERM_\DIS$ maximizes $s_t$.
\end{proof}

The following Theorem exhibits the trade-off between accuracy and rejections.
\begin{theorem}\label{thm:wtp-real-trans}
For any $n \in \nats$, any $\eps\geq 0$,
\begin{equation}\label{eq:wtp-error}
\forall \bx, \tbx\in X^n, f \in C: \err_{\tbx}(h|_S) \leq \eps,
\end{equation}
where $S=\wtp(\bx, \tbx, \eps)$ and $h=\ERM_C(\bx,f(\bx))$.
Furthermore, for any $\delta > 0$ and any distribution~$P$ over~$X$:
\begin{equation}\label{eq:wtp-rej}
\Pr_{\bx,\bz\sim P^n}\left[\rej_\bz(S)\leq \frac{1}{n}\left(\frac{2d\log 2n}{\eps} + \log 1/\delta\right) \right]\geq 1-\delta.
\end{equation}
\end{theorem}

Before we prove \Cref{thm:wtp-real-trans} we provide some generalization bounds that will be used in the proof.
To this end, given a family $G$ of classifiers $g: X \rightarrow \{0,1\}$, following \cite{Blumer89}, define:
\begin{equation}
\Pi_G[2n] \defeq \max_{\bw \in X^{2n}}|\{g(\bw): g\in G\}|.
\end{equation}
\begin{lemma}[Transductive train-test bounds]\label{lemma:trans:aux}
For any $n \in \nats$, any distribution $P$ over a domain~$X$, any set $G$ of classifiers over $X$, and any $\epsilon>0$,
\begin{equation}\label{eq:trans1}
\Pr_{\bx,\bz\sim P^n}\left[\exists g \in G: ~\left(\frac{1}{n}\sum_i g(z_i) \geq \eps\right)\wedge\left(\frac{1}{n}\sum_i g(x_i)=0\right)\right]\leq \Pi_G[2n]2^{-\eps n}
\end{equation}
and
\begin{equation}\label{eq:trans2}
\Pr_{\bx,\bz\sim P^n}\left[\exists g \in G: ~\frac{1}{n}\sum_i g(z_i)\geq \frac{1+\alpha}{n}\sum_i g(x_i) + \eps \right]\leq \Pi_G[2n]e^{-\frac{2\alpha}{(2+\alpha)^2}\eps n}.
\end{equation}
\end{lemma}
The proof of this lemma is deferred to \Cref{sec:gen}. (Note \cref{eq:trans2} is used for the agnostic analysis later.)

\begin{proof}[Proof of \Cref{thm:wtp-real-trans}]
We denote for $T\geq 1$ and classifier vectors $\bc,\bc'\in C^T$:
\begin{align*}
\delta_{\bc,\bc'}(x)&\defeq\max_{i\in[T]} \dis_{c_i,c'_i}(x)=\begin{cases}1 & \text{if }c_i(x)\neq c'_i(x) \text{ for some }i \in [T]\\
0&\text{otherwise}.\end{cases}\\
\Delta_T&\defeq\left\{\delta_{\bc,\bc'}:~\bc,\bc'\in C^T\right\}.
\end{align*}
Thus the output of $\wtp$ is $S_{T+1}=\{x\in X:~\delta_{\bc,\bc'}(x)=0\}$ for the vectors $\bc=(c_1,\ldots,c_T)$ and $\bc'=(c'_1,\ldots,c'_T)$ chosen by the algorithm.

Let $T$ be the final iteration of the algorithm so that the output of the algorithm is $S=S_{T+1}$.
Note that $\err_\bx(f,h)=0$, by definition of $\ERM_C$, so   $s_{T+1}(f,h)=\err_{\tbx}(h|_S) \leq \eps$ (otherwise the algorithm would have chosen $c=h, c'=f$ instead of halting) which implies \cref{eq:wtp-error}.

By Lemma  \ref{lem:eff-wtp}, WLOG we can take $T=\lfloor 1/\eps \rfloor$ by padding with classifiers $c_t=c'_t$.

We next claim that $x_i \not\in S_t$ for all $i\in [n]$, i.e., $\delta_{\bc,\bc'}(x_i)=0$. This is because the algorithm is run with $\Lambda = n+1$, so any disagreement $c_t(x_i)\neq  c'_t(x_i)$ would result in a negative score $s_t(c_t, c'_t)$.  (But a zero score is always possible by choosing $c_t=c_t'$.) Thus we must have the property that $\dis_{c_t',c_t}(x_i)=0$ and hence $\delta_{\bc,\bc'}(x_i)=0$. Now, it is not difficult to see that 
$\Pi_{\Delta_T}[2n]\leq (2n)^{2d/\eps}$ because, by Sauer's lemma, there are at most $N=(2n)^d$ different labelings of $2n$ examples by classifiers from $C$, hence there are at most ${N \choose 2}^T \leq (2n)^{2dT}$ disagreement labelings for $T\leq 1/\eps$ pairs. Thus for $\xi = \frac{1}{n}\left(\frac{2d\log 2n}{\eps} + \log 1/\delta\right)$, by \Cref{lemma:trans:aux},
\[\Pr_{\bx,\bz\sim P^n}\left[\forall g \in \Delta_T \text{ s.t. } \sum_i g(x_i)=0:~\frac{1}{n}\sum_i g(z_i)\leq \xi\right]\geq 1-\Pi_{\Delta_T}[2n]2^{-\xi n}\geq 1-\delta.\]
If this $1-\delta$ likely event happens, then also $\rej_\bz(S)=\frac{1}{n}\sum_i \delta_{\bc,\bc'}(z_i)\leq \xi$ for the algorithm choices $\bc,\bc'$.
\end{proof}

\begin{proof}[Proof of Theorem~\ref{cor:wtp-real-trans}]
The proof follows from \Cref{thm:wtp-real-trans}  and the fact that,
\[\frac{1}{n}\left(\frac{2d \log 2n}{\eps^*} + \log 1/\delta\right) \leq \frac{2d \log 2n}{n\sqrt{\frac{2d\log 2n}{n}}} + \frac{\log 1/\delta}{n} = \eps^*.\]
\end{proof}

\section{Massart Noise}\label{ap:massart}
This section shows that we can PQ learn in the Massart noise model. The Massart model \citep{massart2006risk} is defined with respect to a noise rate $\eta< 1/2$ and function (abusing notation) $\eta : X \rightarrow [0,\eta]$:
\begin{definition}[Massart Noise Model]
Let $P$ be a distribution on $X$, $\eta<1/2$, and $0 \leq \eta(x) \leq \eta$ for all $x \in X$. The Massart distribution $P_{\eta,f}$ with respect to $f$ over $(x,y) \in X \times Y$ is defined as follows: first $x\sim P$ is chosen and then $y=f(x)$ with probability $1-\eta(x)$ and $y=1-f(x)$ with probability $\eta(x)$.
\end{definition}
When clear from context, we omit $f$ and write $P_\eta=P_{\eta,f}$. The following lemma relates the \textit{clean} error rate $\err_P(h, f) = \Pr_P[h(x)\neq f(x)]$ and \textit{noisy} error rate  $\err_{P_\eta} = \Pr_{(x,y)\sim P_\eta}[h(x)\neq y]$. Later, we will show how to drive the clean error arbitrarily close to 0 using an ERM.
\begin{lemma}
\label{lem:noisy-err}
For any classifier $g:X\to Y$, any $\eta<1/2, f \in C$, and any distribution $P_\eta$ corrupted with Massart noise:
\[(1-2\eta) \err_P(g) \leq \err_{P_\eta}(g) - {\rm OPT},\]
where ${\rm OPT}  = \min_{h\in C} \err_{P_\eta}(h) = \E_{x \sim P}[\eta(x)]$.
\end{lemma}

\begin{proof}
By definition of the noisy error rate of $g$ under $P_\eta$, observe the following:
\begin{align*}
    \err_{P_\eta}(g) &= \Pr_{(x, y)\sim P_\eta}[g(x)\neq y]\\
                       &= \E_{x\sim P} \left[ \eta(x) \mathbf{1}\{g(x)=f(x)\} + (1-\eta(x)) \mathbf{1}\{g(x)\neq f(x)\} \right]\\
                       &= \E_{x\sim P} \left[ \eta(x) (1 - \mathbf{1}\{g(x)\neq f(x)\}) + (1-\eta(x)) \mathbf{1}\{g(x)\neq f(x)\} \right]\\
                       &= \E_{x\sim P} \left[ \eta(x)\right] + \E_{x\sim P} \left[ (1-2\eta(x))\mathbf{1}\{g(x)\neq f(x)\}\right]\\
                       &= {\rm OPT} + \E_{x\sim P} \left[ (1-2\eta(x))\mathbf{1}\{g(x)\neq f(x)\}\right]\\
                       &\geq {\rm OPT} + (1-2\eta) \E_{x\sim P} [\mathbf{1}\{g(x)\neq f(x)\}]\\
                       &= {\rm OPT} + (1-2\eta) \err_P(g),
\end{align*}
where the last inequality follows from the fact that $\eta(x)\leq \eta$ for every $x\in X$. Rearranging the terms concludes the proof.
\end{proof}
The following lemma shows that using an extra $N=\tilde{O}\left(\frac{dn^2}{\delta^2(1-2\eta)^2}\right)$ i.i.d. examples $(\bx',\by')\sim P^N_{\eta}$, we can ``denoise'' the $n$ held-out examples $(\bx,\by)\sim P^{n}_{\eta}$ with $\hat{h}=\ERM_C(\bx',\by')$, and then run $\wwtp$ on $(\bx, \hat{h}(\bx))$. This shows that we can PQ learn $C$ under Massart noise. 
\begin{lemma}[Massart denoising]
\label{lem:denoise-massart}
For any $f\in C$ and any distribution $P$ over $X$, any $\eta <1/2$ and $\eta: X \rightarrow  [0, \eta]$, let $P_\eta$ be the corresponding Massart distribution over $(x,y)$. For any $n\in \nats$, let $(\bx,\by)= (x_1,y_1),\dots, (x_n,y_n)\sim P_{\eta}$ be i.i.d. examples sampled from $P_\eta$. Then, 
\[\Pr_{(\bx',\by')\sim P^N_{\eta}} \left[ \err_{\bx} (\hat{h}, f) = 0 \right] \geq 1-\delta,\]
where $\hat{h} = \ERM_C(\bx',\by')$ and $N=O\left(\frac{d n^2 + \log(2/\delta)}{\delta^2(1-2\eta)^2}\right)$.
\end{lemma}

\begin{proof}
By agnostic learning guarantees for $\ERM_{C}$, we have that for any $\eps',\delta>0$:
\[\Pr_{(\bx',\by')\sim P^N_\eta} \left[ \err_{ P_{\eta}}(\hat{h}) \leq {\rm OPT} + \eps' \right] \geq 1-\frac{\delta}{2},\]
where $\hat{h} = \ERM_C(\bx',\by')$ and $N=O(\frac{d+\log(2/\delta)}{\eps'^2})$.
By \Cref{lem:noisy-err}, choosing $\eps' = \frac{\delta}{2n} (1-2\eta)$ guarantees that the clean error rate $\err_P(\hat{h})\leq \frac{\delta}{2n}$. Since, $(\bx,\by)\sim P^n_{\eta}$ are independent held-out examples, by a union bound, we get that $\err_{\bx}(\hat{h},f)=0$ with probability $1-\delta$. 
\end{proof}
This yields an easy algorithm and corollary: simply use the $N$ examples $\bx', \by'$ to denoise the $n$ labels for $\bx$ and then run the $\wwtp$ algorithm. 
\begin{corollary} [PQ guarantees under Massart noise] \label{corr:massart}
For any $n \in \nats, \delta>0, f \in C$ and distributions $P,Q$ over $X$, any $\eta <1/2$ and $\eta: X \rightarrow  [0, \eta]$, let $P_\eta$ be the corresponding Massart distribution over $(x,y)$. Then,
\[\Pr_{(\bx',\by')\sim P^N_\eta, (\bx,\by)\sim P^n_\eta, \tbx\sim Q^n}[\err_Q\leq 2\eps^*  ~\wedge~\rej_P  \leq \eps^*]\geq 1-\delta,\]
where $\eps^* = \sqrt{8\frac{d \ln 2n}{n}}+\frac{8 \ln 32/\delta}{n}$, $N=O\left(\frac{d n^2 + \log(2/\delta)}{\delta^2(1-2\eta)^2}\right)$, $\hat{h}=\ERM_C(\bx',\by')$, and $h|_S=\wwtp(\bx,\hat{h}(\bx),\tbx,\eps^*)$.
\end{corollary}


\section{Semi-agnostic analysis}\label{ap:ag}

In agnostic learning, the learner is given pairs $(x,y)$ from some unknown distribution~$\mu$, and it is assumed that  there exists some  (unknown) $f \in C, \eta \geq 0$ such that $$\err_\mu(f)\defeq \Pr_{(x,y)\sim\mu}[y\neq f(x)]\leq \eta.$$ In this work, we consider the case where the test distribution $\tilde{\mu}$ may be (arbitrarily) different from the train distribution~$\mu$, yet we require the existence of parameters $\eta, \tilde\eta \geq 0$ and an (unknown) $f \in C$ such that  $$\err_\mu(f)\leq \eta~\mbox{ and }~\err_{\tmu}(f)\leq \tilde\eta.\footnote{\mbox{In \Cref{sec:main}, we assumed that $\eta=\tilde\eta$ for simplicity, yet here we consider the more general case where $\eta$ and $\tilde\eta$ may differ.}}$$

Moreover, in this work we assume that $\eta$ and $\tilde{\eta}$ are known.  Unfortunately, even with this additional assumption, agnostic learning is challenging when $\mu\neq\tmu$ and one cannot achieve guarantees near $\max\{\eta,\tilde{\eta}\}$ as one would hope, as we demonstrate below.

In what follows, we slightly abuse notation and write $(\bx,\by)\sim D^n$ to denote $(x_i, y_i)$ drawn iid from $D$ for $i=1,2,\ldots,n$. The definitions of error and rejection with respect to such a distribution are:
\begin{align*}
\rej_D(S) &\defeq \Pr_{(x, y)\sim D}[x \not\in S]\\
\err_D(h|_S) &\defeq \Pr_{(x, y)\sim D}[h(x)\neq y \wedge x \in S]
\end{align*}

We prove the following lower bound.
\begin{lemma}\label{ag-lower}
There exists a family of binary classifiers $C$ of VC dimension 1, such that for any $\eta,\tilde\eta\in [0,1/2]$ and $n \geq 1$, and for any selective classification algorithm $L: X^n \times Y^n \times X^n \rightarrow Y^X \times 2^X$ there exists $\mu,\tmu$ over $X\times Y$ and $f \in C$ such that:
\[\E_{\substack{(\bx,\by)\sim \mu^n\\(\tbx,\tilde\by)\sim\tmu^n}}[\err_\tmu(h|_S)+\rej_\mu(S)] \geq \max\left\{\sqrt{\eta/8},~\tilde\eta\right\}.\]
where $h|_S = L(\bx, \by, \bz)$ and where $\err_\mu(f)\leq \eta$ and $\err_{\tmu}(f)\leq \tilde\eta$.
\end{lemma}
The proof is deferred to Section \ref{ap:lower-proofs}.

We now show that $\wwtp$ can be used to achieve nearly this guarantee. Recall that in the realizable setting, we fixed $\Lambda=n+1$ in $\wwtp$. 
In this semi-agnostic setting, we will set $\Lambda$ as a function of $\eta$, hence our learner requires knowledge of $\eta$ unlike standard agnostic learning when $\mu=\tilde\mu$.
\begin{theorem}[Agnostic generalization]\label{wow:ag-gen}
For any $n\in \nats$, any $\delta,\gamma \in (0,1)$, any $\eta,\tilde\eta\in [0,1)$, and any distributions $\mu, \tmu$ over $X\times Y$ such that that $\err_\mu(f)\leq \eta$  and $\err_{\tmu}(f)\leq \tilde\eta$ simultaneously for some $f\in C$:
\[\Pr_{\substack{(\bx,\by) \sim \mu^n\\(\tbx,\tilde\by) \sim \tmu^n}}\left[\left(\err_{\tmu}(h|_S) \leq 2\sqrt{2\eta} +2\tilde\eta + 4\eps^*\right) \wedge \left(\rej_\mu(S)  \leq 4\sqrt{2\eta}+4\eps^*\right)\right]\geq 1-\delta,\]
where $\eps^* = 4\sqrt{\frac{d \ln 2n+\ln 48/\delta}{n}}$, $\Lambda^* = \sqrt{\frac{1}{8\eta+(\eps^*)^2}}$, and $h|_S= \wwtp(\bx,\by,\tbx,\eps^*,\Lambda^*)$.
\end{theorem}
A few points of comparison are worth making:
\begin{itemize}
    \item When $\eta=\tilde\eta=0$, one recovers guarantees that are slightly worse than those in the realizable (see \Cref{wow:real-gen}).
    \item In standard agnostic learning, where $\mu$ and $\tmu$ are identical, and thus $\eta=\tilde\eta$, one can set $S=X$ (i.e., select everything) and ERM guarantees $\err\bigl(h|_S(\tbx), \tilde\by\bigr)\leq \eta + \eps$ w.h.p.\ for $n$ sufficiently large.
    \item The above theorem can be used to bound $\rej_\tmu$ using the following lemma:
\end{itemize}

\begin{lemma}\label{lem:convert2}
For any $S\subseteq X$, $f, h \in Y^X$ and distributions $\mu, \tmu$ over $X\times Y$:
\[\rej_{\tmu}(S) \leq \rej_\mu(S) + |\mu_X-\tmu_X|_\TV\leq \rej_\mu(S) + |\mu-\tmu|_\TV,\]
where $\mu_X, \tmu_X$ are the marginal distributions of $\mu, \tmu$ over $X$.
\end{lemma}
\begin{proof}
The lemma follows from \Cref{lem:convert} applied to $P=\mu_X, Q=\tmu_X$, and from the fact that the total variation between marginal distributions is no greater than the originals: $|\mu_X-\tmu_X|_\TV\leq |\mu-\tmu|_\TV$.\end{proof}

As before, it will be useful (and easier) to first analyze the transductive case. In this case, it will be useful to further abuse notation and define, for any $\by, \by' \in \{0,1, \question\}^n$,
\[\err(\by, \by')\defeq \frac{1}{n}\bigl|\{i:y_i=1-y'_i\}\bigr|.\]
Using this, we will show:
\begin{theorem}[Agnostic transductive] \label{thm:ag-trans}
For any $n\in\mathbb{N}$, $\eps,\delta,\Lambda\geq 0$, $f \in C$:
\begin{equation}\label{eqn:ag-trans-err}
\forall \bx,\tbx\in X^n, \by,\tilde\by \in Y^n:~\err(h|_{S}(\tbx), \tilde\by) \leq \eps +  2\Lambda\cdot \err(f(\bx), \by) +  \err(f(\tbx),\tilde\by),
\end{equation}
where $h|_S= \wwtp(\bx,\by,\tbx,\eps,\Lambda)$.
Furthermore,
\begin{equation}\label{eqn:ag-trans-rej}
\Pr_{\bx,\bz\sim P^n}\left[\forall\by\in Y^n, \tbx\in X^n:~\rej_\bz\bigl(S)\leq 2\Lambda^{-1} + \frac{9}{n}\left(\frac{d \ln 2n}{\eps} +\frac{\ln 1/\delta}{2}\right) \right]\geq 1-\delta.
\end{equation}
\end{theorem}
The above bounds suggest the natural choice of $\Lambda = \eta^{-1/2}$ if $\err(f(\bx), \by)\approx \eta$. The following two Lemmas will be used in its proof.
\begin{lemma}\label{lem:low-abstain}
For any $n\in\mathbb{N}$, $\eps,\Lambda\geq 0$,  $\bx,\bz\in X^n, \by\in Y^n$: $\rej_\bx(S) \leq 1/\Lambda$ where $h|_S= \wwtp(\bx,\by,\bz,\eps,\Lambda)$.
\end{lemma}
\begin{proof}
Note that for each iteration~$t$ of the algorithm $\wwtp(\bx,\by,\bz,\eps,\Lambda)$,
\[\sum_{i\in[n]:z_i \in S_t}\left|c_t(z_i)-h(z_i)\right|-\Lambda\sum_{i\in[n]} \left|c_t(x_i)-h(x_i)\right|\geq 0,\]
because $c_t$ maximizes the above quantity over $C$, and that quantity is zero at $c_t=h \in C$. Also note that $x \not\in S$ if and only if $|c_t(x)-h(x)|=1$ for some $t$. More specifically, for each $i\in[n]$ such that $z_i \not\in S$ there exists a {\em unique} $t\in[T]$ such that $z_i\in S_{t}$, and yet $z_i\notin S_{t+1}$, where the latter occurs when  $|c_{t}(z_i)-h(z_i)|=1$. Thus the total number of test and train rejections can be related as follows:
\[n \geq n \rej_\bz(S) = \sum_{t\in[T]}\sum_{i\in[n]: z_i \in S_t} \left| c_{t}(z_i)-h(z_i)\right|
 \geq \sum_{t\in[T]}\Lambda\sum_{i\in [n]} \left|c_{t}(x_i)-h(x_i)\right|\geq n\Lambda \rej_\bx(S).\]
Dividing both sides by $n\cdot \Lambda$ gives the lemma.
\end{proof}
The following lemma is proven in \Cref{sec:gen}.
\begin{lemma}\label{lem:lol2}
For any $T,n\in\mathbb{N}$, any $\delta\geq 0$, and $\eps=\frac{9}{2n}\left(d(T+1)\ln(2n) + \ln \frac{1}{\delta}\right)$:
\[\Pr_{\bx, \bz\sim P^n}\left[\exists \bc \in C^T, h \in C:~\rej_\bz({S(h,\bc)})>2\rej_\bx({S(h,\bc)})+\eps \right] \leq \delta.\]
\end{lemma}

Using these, we can now prove the transductive agnostic theorem.
\begin{proof}[Proof of \Cref{thm:ag-trans}]
To prove Equation~\eqref{eqn:ag-trans-err}, first fix any $\bx,\tbx\in X^n, \by,\tilde\by \in Y^n, f\in C$. Since $f\in C$ the output $h=\ERM_C(\bx,\by)$ satisfies $\err(h(\bx), \by)\leq \err(f(\bx), \by)$. By the triangle inequality, this implies that
\begin{equation}\label{eqn:ag:err-x}
\err_\bx(h, f) =  \frac{1}{n}\sum_{i \in [n]} |h(x_i)-f(x_i)| \leq \frac{1}{n}\sum_{i \in [n]} \bigl(|h(x_i)-y_i|+|y_i-f(x_i)|\bigr)\leq 2\err(f(\bx),\by).
\end{equation}

Now suppose the algorithm $\wwtp$ terminates on iteration~$T+1$ and the output is $h|_S$ for $S=S_{T+1}$.  Then by definition, for every $c\in\functions$, \[s_{T+1}(c)=\err_{\tbx}(h|_{S}, c) - \Lambda\cdot \err_\bx(h, c) \leq  \eps,\]
For $c=f$ in particular,
\begin{equation*}
\err_{\tbx}(h|_{S},f) \leq  \Lambda\cdot \err_\bx(h, f) + \eps \leq 2\Lambda\cdot\err(f(\bx),\by) + \eps.
\end{equation*}
\Cref{eqn:ag-trans-err} follows from the above and the fact that
\[\err(h|_S(\tbx),\tilde\by) \leq \err(h|_{S_T}(\tbx), f(\tbx))+ \err(f(\tbx),\tilde\by).\]

We next prove \cref{eqn:ag-trans-rej}. 
By \Cref{lem:low-abstain}, $\rej_\bx(S)\leq 1/\Lambda$ with certainty. So by \Cref{lem:lol2} applied to the marginal distribution $P=\mu_X$ over $X$,
\[\Pr_{\bx,\bz\sim P^n}\left[\exists h\in C, \bc\in C^T:~\rej_\bz({S(h,\bc)}) > 2\rej_\bx({S(h,\bc)})+\xi \right]\leq \delta,\]
for $\xi=\frac{9}{2n}\left(\frac{2d}{\eps} \ln(2n) + \ln \frac{1}{\delta}\right)$ using $T+1\leq 2/\eps$. This implies \cref{eqn:ag-trans-rej}.
\end{proof}


Returning to the generalization (distributional) agnostic case,  the following theorem shows the trade-off between error and rejections:
\begin{theorem}\label{thm:ag-gen}
For any $n\in\mathbb{N}$ and $\delta, \Lambda\geq 0$, any $\eps\geq 4\sqrt{\frac{d\ln 2n + \ln 24/\delta}{n}}$,  any $f\in\functions$, and any distributions $\mu,\tilde{\mu}$ over $X\times Y$:
\begin{equation}\label{eqn:ag-gen-err}
\Pr_{\substack{(\bx,\by) \sim \mu^n\\(\tbx,\tilde\by) \sim \tmu^n}}\left[\err_{\tmu}(h|_S) \leq 8\Lambda \err_\mu(f) +2\err_{\tmu}(f) +\Lambda \eps^2 +3\eps\right]
\geq 1-\delta,
\end{equation}
where $h|_S= \wwtp(\bx,\by,\tbx,\eps,\Lambda)$.
Furthermore,
\begin{equation}\label{eqn:ag-gen-rej}
\Pr_{(\bx,\by) \sim \mu^n}\left[\forall \tbx \in X^n:~\rej_\mu(S)  \leq \frac{2}{\Lambda}+ 2\eps\right]\geq 1-\delta.
\end{equation}
\end{theorem}

The proof of this theorem will use the following lemma.
\begin{lemma}\label{lem:annoying}
For any $h \in C$, distribution $\mu$ over $X \times Y$, and $\eps = \frac{16}{n}\left(dT \ln 2n + \ln \frac{8}{\delta}\right),$
\begin{equation*}
\Pr_{(\bx,\by) \sim \mu^n}\left[\forall \bc \in C^T:~\err_{\mu}(h|_{S(h,\bc)}) \leq \max\left\{2\err(h|_{S(h,\bc)}(\bx),\by),\eps\right\}\right]\geq 1-\delta.
\end{equation*}
\end{lemma}
The proof of this lemma is deferred to \Cref{sec:gen}.

\begin{proof}[Proof of \Cref{thm:ag-gen}]
The proof structure follows the proof of \Cref{thm:real-gen}.
Fix $f$. We start by proving Equation~\eqref{eqn:ag-gen-err}.
Let $\zeta = \frac{16}{n}\left(\frac{2d}{\eps} \ln 2n + \ln \frac{24}{\delta}\right).$  By \Cref{lem:annoying},
\begin{equation*}
\forall h \in C~~\Pr_{(\tbx,\tilde\by) \sim \tilde\mu^n}\left[\forall \bc \in C^T:~\err_{\tilde\mu}(h|_{S(h,\bc)}) \leq \max\left\{2\err(h|_{S(h,\bc)}(\tbx),\tilde\by),\zeta\right\}\right]\geq 1-\delta/3.
\end{equation*}

\Cref{eqn:ag-trans-err} from \Cref{thm:ag-trans} states that,
\[
\forall \bx,\tbx\in X^n, \by,\tilde\by \in Y^n:~\err(h|_{S}(\tbx), \tilde\by) \leq  2\Lambda\cdot \err(f(\bx), \by) +  \err(f(\tbx),\tilde\by) + \eps,
\]
with certainty. We next bound $\err(f(\bx), \by)$ and $\err(f(\btx), \tilde\by).$


Since $\eps^2/4 \geq \frac{4}{n}\ln \frac{3}{\delta}$, multiplicative Chernoff bounds imply that,
\begin{equation*}
\Pr_{(\bx,\by) \sim \mu^n}\left[\err(f(\bx),\by)\leq 2\err_\mu(f)+ \frac{\eps^2}{4}\right]\geq 1-\delta/3.
\end{equation*}
Also, since $\eps/2 \geq \sqrt{\ln (3/\delta)/(2n)}$, additive Chernoff bounds imply that,
\begin{equation*}
\Pr_{(\tbx,\tilde\by) \sim \tilde\mu^n}\left[\err(f(\tbx),\tilde\by)\leq \err_{\tilde\mu}(f)+\frac{\eps}{2} \right] \geq 1-\delta/3
\end{equation*}
 Combining previous four displayed inequalities, which by the union bound all hold with probability $\geq 1-\delta$, gives,
\begin{equation*}
\Pr_{\substack{(\bx,\by) \sim \mu^n\\(\tbx,\tilde\by) \sim \tmu^n}}\left[\err_{\tmu}(h|_S) \leq \max\left\{2\left(2\Lambda (2\err_\mu(f) +\eps^2/4) +  (\err_{\tmu}(f)+\eps/2) + \eps\right),\zeta\right\}\right]
\geq 1-\delta.
\end{equation*}
Simplifying:
\begin{equation}\label{eqn:912}
\Pr_{\substack{(\bx,\by) \sim \mu^n\\(\tbx,\tilde\by) \sim \tmu^n}}\left[\err_{\tmu}(h|_S) \leq \max\left\{8\Lambda \err_\mu(f) +\Lambda \eps^2 +  2\err_{\tmu}(f) + 3\eps,\zeta\right\}\right]
\geq 1-\delta.
\end{equation}
Next, we note that for our requirement of $\eps\geq 4\sqrt{\frac{d\ln 2n + \ln 24/\delta}{n}}$, $\zeta \leq 2\eps$ because:
\[\zeta= \frac{16}{n}\left(\frac{2d}{\eps} \ln 2n + \ln \frac{24}{\delta}\right) \leq
\frac{32}{n\eps}\left(d \ln 2n + \ln \frac{24}{\delta}\right) \leq 2\frac{\eps^2}{\eps}.\]
Thus we can remove the maximum from \cref{eqn:912},
\[
\Pr_{\substack{(\bx,\by) \sim \mu^n\\(\tbx,\tilde\by) \sim \tmu^n}}\left[\err_{\tmu}(h|_S) \leq 8\Lambda \err_\mu(f) +\Lambda \eps^2 + 2 \err_{\tmu}(f) + 3\eps\right]
\geq 1-\delta,
\]
which is equivalent to what needed to prove in \cref{eqn:ag-gen-err}.

We next prove \cref{eqn:ag-gen-rej}. By \Cref{lem:low-abstain}, $\rej_\bx(S)\leq 1/\Lambda$ with certainty. So by \Cref{lemma:gen:aux} (\Cref{eqn:Blumer2}) with $\gamma=1/2$,
\[\Pr_{\bx,\bz\sim P^n}\left[\exists h\in C, \bc\in C^T:~\rej_\mu\bigl({S(h,\bc)}\bigr) > 2\rej_\bx\bigl({S(h,\bc)}\bigr)+\xi \right]\leq \delta,\]
for $\xi=\frac{16}{n}\left(\frac{2d}{\eps} \ln(2n) + \ln \frac{8}{\delta}\right)$ using $T+1\leq 2/\eps$. This implies \cref{eqn:ag-gen-rej} using the fact that,
\[\xi=\frac{16}{n}\left(\frac{2d}{\eps}\ln(2n) + \ln \frac{16}{\delta}\right) \leq 2 \cdot \frac{16}{n\eps}\left(d \ln (2n) + \ln \frac{16}{\delta}\right)\leq 2 \cdot \frac{\eps^2}{\eps} = 2\eps.\]
\end{proof}

From this theorem, our main agnostic upper-bound follows in a  straightforward fashion.
\begin{proof}[Proof of \Cref{wow:ag-gen}]
Note that for our choice of  $\Lambda^* = \sqrt{\frac{1}{8\eta+(\eps^*)^2}}$,
\begin{align*}
\left(8\Lambda^* \err_\mu(f) +2\err_{\tmu}(f)\right) +\Lambda^* (\eps^*)^2 +3\eps^*&\leq  \Lambda^* (8\eta+(\eps^*)^2) + 2 \tilde\eta + 3 \eps^*\\
&= \sqrt{8\eta+(\eps^*)^2} + 2 \tilde\eta + 3\eps^* \\
&\leq 2\sqrt{2\eta}+\eps^* + 2 \tilde\eta + 3\eps^*,
\end{align*}
using the fact that $\sqrt{a+b}\leq \sqrt{a}+\sqrt{b}$.
For the chosen $\eps^*= 4\sqrt{\frac{d\ln 2n + \ln 48/\delta}{n}}$, \Cref{thm:ag-gen} implies,
\begin{align}
\Pr_{\substack{(\bx,\by) \sim \mu^n\\(\tbx,\tilde\by) \sim \tmu^n}}\left[\err_{\tmu}(h|_S) \leq \left(8\Lambda^* \err_\mu(f) +2\err_{\tmu}(f)\right) +\Lambda^* (\eps^*)^2 +3\eps^*\right] \nonumber
&\geq 1-\delta/2\\
\Pr_{\substack{(\bx,\by) \sim \mu^n\\(\tbx,\tilde\by) \sim \tmu^n}}\left[\err_{\tmu}(h|_S) \leq 2\sqrt{2\eta} + 2\tilde\eta +4\eps^*\right]
&\geq 1-\delta/2\label{eq:oo1}
\end{align}
Also note that
\[\frac{2}{\Lambda^*}+2\eps^*  \leq 2\sqrt{8\eta+(\eps^*)^2} +2\eps^* \leq 4\sqrt{2\eta}+2\eps^* + 2 \eps^* \leq 4\sqrt{2\eta} + 4 \eps^*.\]
\Cref{thm:ag-gen} also implies:
\begin{align}
\Pr_{(\bx,\by) \sim \mu^n}\left[\forall \tbx \in X^n:~\rej_\mu(S)  \leq \frac{2}{\Lambda^*}+ 2\eps^*\right]&\geq 1-\delta/2\nonumber\\
\Pr_{(\bx,\by) \sim \mu^n}\left[\forall \tbx \in X^n:~\rej_\mu(S)  \leq 4\sqrt{2\eta} + 4 \eps^*\right]&\geq 1-\delta/2\nonumber\\
\Pr_{(\bx,\by) \sim \mu^n}\left[\forall \tbx \in X^n:~\rej_{\tmu}\bigl(S\bigr)  \leq 4\sqrt{2\eta} + 4 \eps^* + |\mu-\tmu|_\TV\right]&\geq 1-\delta/2,
\label{eq:oo2}
\end{align}
where we have used \Cref{lem:convert2} in the last step. The union bound over \cref{eq:oo1} and \cref{eq:oo2} proves the corollary.
\end{proof}

\section{Generalization Lemmas}\label{sec:gen}

In this section we state auxiliary lemmas that relate the empirical error and rejection rates to generalization error and rejection rates.

To bound generalization, it will be useful to note that the classifiers $h|_S$ output by our algorithm are not too complex. To do this, for any $k \in \nats$ and any classifiers $c_1, c_2, \ldots, c_k:X \rightarrow Y$, define the \textit{disagreement} function that is 1 if any of two disagree on $x$:
\begin{equation}
    \dis_{c_1,\ldots, c_k}(x)\defeq \begin{cases}0 & \text{ if } c_1(x)=c_2(x)=\cdots=c_k(x)\\1 &                                    \text{otherwise}.\end{cases}   \label{eq:phi}
\end{equation}
Also denote by $\bar{f}=1-f$ and $\bc=(c_1,\ldots,c_T)\in C^T$. In these terms we can write,
\begin{align*}
    \dis_{h,\bc}&= \begin{cases}0 & \text{ if } h(x)=c_1(x)=c_2(x)=\cdots=c_T(x)\\1 &                               \text{otherwise}\end{cases}\\
    \dis_{\bar{f}, h, \bc} &= \begin{cases}1 & \text{ if } 1-f(x)=h(x)=c_1(x)=c_2(x)=\cdots=c_T(x)\\0 &                                   \text{otherwise}.\end{cases}
\end{align*}

Recall the definition of $\Pi_G[2n]$ for a family $G$ of classifiers $g: X \rightarrow \{0,1\}$:
\begin{equation*}
\Pi_G[2n] \defeq \max_{\bw \in X^{2n}}|\{g(\bw): g\in G\}|.
\end{equation*}
\begin{lemma}[Complexity of output class]\label{lem:complexity}
For any $h\in C$, let 
\begin{equation}
    \Dis_{T} \defeq \left\{\dis_{h,c_1,\ldots, c_T}: h,c_1,\ldots, c_T \in C\right\} \label{eq:PhifT}
\end{equation}
\begin{equation}
\Dis_{h,T} \defeq \left\{\dis_{h,c_1,\ldots, c_T} : c_1,\ldots, c_T \in C\right\} \label{eq:PhifhT},\\
\end{equation}
\begin{equation}
\Dis_{f,h,T} \defeq \left\{\dis_{f,h,c_1,\ldots, c_T} : c_1,\ldots, c_T \in C\right\} \label{eq:PhifhT-f},\\
\end{equation}

Then $\Pi_{\Dis_{T}}[2n] \leq (2n)^{d(T+1)}$, $\Pi_{\Dis_{h,T}}[2n]\leq (2n)^{dT}$, and $\Pi_{\Dis_{f,h,T}}[2n]\leq (2n)^{dT}$, where $d$ is the VC dimension of $C$.  
\end{lemma}
\begin{proof}
The proof follows trivially from Sauer's lemma, since the number of labelings of $2n$ examples by any $c\in C$ is at most $(2n)^d$ and there are $T$ choices of $c_1, \ldots, c_T$ and 1 choice of $h$.
\end{proof}

\begin{lemma}[Generalization bounds using \cite{Blumer89}]\label{lemma:gen:aux}
For any $n \in \nats$, any distribution $P$ over a domain~$X$, any set $G$ of binary classifiers over $X$, and any $\epsilon>0$,
\begin{equation}\label{eqn:Blumer1}
\Pr_{\bz\sim P^n}\left[\exists g \in G: ~\left(\E_{x\sim P}[g(x)]>\eps\right)\wedge \left(\frac{1}{n}\sum_{i\in [n]} g(z_i) =0\right) \right]\leq 2 \Pi_G[2n]2^{-\eps n/2},
\end{equation}
and, for any $\gamma \in (0,1),$
\begin{equation}\label{eqn:Blumer2}
\Pr_{\bz\sim P^n}\left[\exists g \in G: \E_{x\sim P}[g(x)]>\max\left\{\eps, \frac{1}{1-\gamma}\cdot \frac{1}{n}\sum_{i\in [n]} g(z_i)\right\}\right]\leq 8 \Pi_G[2n]e^{-\frac{\gamma^2\eps n}{4}}.
\end{equation}
Finally, for any distribution $\mu$ over $X\times Y$ and any $\gamma \in (0,1),$
\begin{equation}\label{eqn:Blumer3}
\Pr_{(\bx,\by)\sim \mu^n}\left[\exists g \in G: \err_\mu(g)>\max\left\{\eps, \frac{1}{1-\gamma}\cdot \frac{1}{n}\sum_{i\in [n]} |g(x_i)-y_i|\right\}\right]\leq 8 \Pi_G[2n]e^{-\frac{\gamma^2\eps n}{4}}.
\end{equation}
\end{lemma}
\begin{proof}
Simply consider a binary classification problem where the target classifier is the constant 0 function, with training examples $\bz \sim P^n$. Then the training error rate is $\sum g(z_i)/n$ and the generalization error is $\Pr_P[g(x)=1]$. By Theorem A2.1 of \cite{Blumer89}, the probability that any $g \in G$ has 0 training error and test error greater than $\eps$ is at most $\Pi_G[2n]2^{-\eps n/2}$. Similarly \cref{eqn:Blumer2} and (\ref{eqn:Blumer3}) follow from Theorem A3.1 of \cite{Blumer89}, noting that the bound holds trivially for all $g$ with $\E[g(x)]\leq \eps$.
\end{proof}
We now prove  \Cref{lem:annoying}, which adapts the last bound above to the agnostic setting.
\begin{proof}[Proof of \Cref{lem:annoying}]
We would like to apply the last inequality of \Cref{lemma:gen:aux} with $\gamma=1/2$, but unfortunately that lemma does not apply to error rates of selective classifiers. First, consider the case where the distribution is ``consistent'' in that $\Pr_{x,y\sim \mu}[y=\tau(x)]$ for some arbitrary $\tau:X\rightarrow \{0,1\}$. We can consider the modified functions,
\[
g_{h, \bc}(x)=\begin{cases}\tau(x) & \text{if $x\not\in S(h,\bc)$}\\
h(x) & \text{otherwise}.
\end{cases}
\]
It follows that $\err_\mu(g_{h,\bc})=\err_\mu(h|_S)$. Furthermore, the class $G=\{g_{h, \bc}:h\in C, \bc \in C^T\}$ satisfies $\Pi_G[2n]\leq (2n)^{dT}$ (just as we argued $\Pi_{\Dis_{h,T}}[2n]\leq (2n)^{dT}$) because there are $(2n)^d$ different labelings of $c$ on $2n$ elements and thus there are at most $(2n)^{d T}$ choices of $T$ of these for $\bc \in C^T$. Thus, applying \Cref{lemma:gen:aux} gives the lemma for consistent $\mu$.

The inconsistent case can be reduced to the consistent case by a standard trick. In particular, we will extend $X$ to $X'=X \times \{0,1\}$ by appending a latent (hidden) copy of $y$, call it $b$, to each example $x$. In particular For $c\in C$, define $c'(x,b)=c(x)$ so that the classifiers cannot depend on $b$. This does not change the VC dimension of the classifiers. However, now, any distribution over $\mu$ can be converted to a consistent distribution $\mu'$ over $X'$ whose marginal distribution over $X$ agrees with $\mu$, by making \[\mu'((x,b),y)=\begin{cases}
\mu(x,y) & \text{if }b=y\\
0 & \text{otherwise}.
\end{cases}\]
In other words, $\Pr_{(x,b),y\sim \mu'}[b=y]=1$. Now, clearly $\mu'$ is consistent. The statement of the lemma applied to $\mu'$ implies the corresponding statement for $\mu$ since the classifiers do not depend on $b$.
\end{proof}

We now prove \Cref{lemma:trans:aux}.
\begin{proof} [Proof of \Cref{lemma:trans:aux}]
Fix any $n \in \nats$, any distribution $P$ over a domain~$X$ and any $\beta\in[n]$. Imagine selecting $\bx, \bz\sim P^n$ by first randomly choosing $2n$ random elements $\bw\sim P^{2n}$ and then randomly dividing these elements into two equal sized sequences $\bx, \bz$. Let $\pi(\bw)$ denote the distribution over the $(2n)!$ such divisions $\bx, \bz \in X^n$.
For any $g \in G$ and $\bw\in X^{2n}$, we claim:
\[
\Pr_{(\bx, \bz)\sim \pi(\bw)}\left[\sum_i g(x_i) = 0 ~\wedge~ \sum_i g(z_i) \geq \lceil \eps n \rceil \right]\leq 2^{-\lceil \eps n \rceil}.
\]
To see this, suppose $s=\sum_i g(w_i)\geq \eps n$ (otherwise the probability above is zero). The probability that all of them are in the test set is at most $2^{-s}\leq 2^{-\eps n}$ because the chance that the first rejection is placed in the test set is $1/2$, the second is $(n-1)/(2n-1) < 1/2$, and so forth. The above equation directly implies \cref{eq:trans1} by dividing by $n$.

We now move to \cref{eq:trans2}. Consider random variables
$A=\sum g(x_i)$ and $B=\sum g(z_i)$. It suffices  to show that that $B > (1+\alpha)A+r$ with probability $\leq  e^{-2\alpha(2+\alpha)^{-2}r}$ for $r=\eps n$. Note that since $B=s-A$,
\[B\geq (1+\alpha)A + r ~\Longleftrightarrow~ A \leq \frac{s-r}{2+\alpha}.\]
Hence, it suffices to prove that
\begin{equation}\label{eqn:pf:lemma:aux:A}
 \Pr\left[A\leq \frac{s-r}{2+\alpha}\right]\leq e^{-\frac{2\alpha}{(2+\alpha)^2}r}.
\end{equation}
Let $\cD$ be the Bernoulli distribution on $\{0,1\}$ with mean $\mu=\frac{s}{2n}$.  Note that by linearly of expectation, $\E[A]=\E[B]=\mu n$. It is well-known that the probabilities of such an unbalanced split are smaller for sampling without replacement than with replacement \citep[see, e.g.,][]{bardenet2015concentration}.  Thus, it suffices to prove Equation~\eqref{eqn:pf:lemma:aux:A} assuming $A$ was sampled by sampling $n$ iid elements $(A_1,\ldots,A_n)\sim\cD^n$, and setting $A=\sum_{i=1}^n A_i$.  By the multiplicative Chernoff bound, 
for every $\rho\in[0,1]$,
\[\Pr\left[A\leq (1-\rho)\mu n\right]\leq e^{-\rho^2 \mu n/2}=e^{-\rho^2 s/4}.\]
In particular, for $\rho=\frac{\alpha s+2r}{s(2+\alpha)}$, since $1-\rho=\frac{2s-2r}{s(2+\alpha)}$ and $\mu n=s/2$, this gives:
\[\Pr\left[A\leq \frac{s-r}{2+\alpha}\right]\leq e^{-\frac{(\alpha s+2r)^2}{4(2+\alpha)^2s}}\]
Hence, it remains to show that the RHS above is at most $\exp\left(-\frac{2\alpha}{(2+\alpha)^2}r\right)$, or equivalently,
\[\frac{(\alpha s+2r)^2}{4(2+\alpha)^2 s} \geq \frac{2\alpha r}{(2+\alpha)^2}.\]
After multiplying both sides by $4(2+\alpha)^2 s$, the above can be rewritten as $(\alpha s+2r)^2 \geq 8\alpha s r$, and equivalently as $(\alpha s-2r)^2\geq 0$, which indeed always holds.
\end{proof}

We are now ready to prove \Cref{lem:lol}.
\begin{proof}[Proof of \Cref{lem:lol}]
Note that for $\dis_{h,\bc}$ defined as in \cref{eq:PhifhT}, $\dis_{h, \bc}(x)=0$ if and only if $x\in S(h,\bc)$. Thus, $\rej_\bx({S(h,\bc)})=0$ implies that
\begin{equation}\label{eq:nonews}
\sum_{i=1}^n \dis_{h,\bc}(x_i)=0.
\end{equation}
Also note that,
\[\rej_\bz(S(h,\bc)) = \frac{1}{n} \sum_{i=1}^n \dis_{h,\bc}(z_i).\]
Hence, it suffices to show
\begin{equation}\label{eq:toshow987}
\Pr_{\bx,\bz\sim P^n}\left[\exists \bc \in C^T, h \in C:~\left(\frac{1}{n}\sum_{i=1}^n \dis_{h,\bc}(x_i)=0\right) ~\wedge~\left(\frac{1}{n}\sum_{i=1}^n \dis_{h,\bc}(z_i)> \eps\right)\right]\leq \delta
\end{equation}
By \Cref{lemma:trans:aux},
\begin{equation}
\label{eq:c0v}
\Pr_{\bx, \bz \sim P^n}\left[\exists \phi \in \Dis_{T}:\left(\sum_{i=1}^n \phi(x_i) = 0\right) \wedge \left(\sum_{i=1}^n \phi(z_i) \geq \eps n\right)  \right] \leq 2^{-\eps n}\Pi_{\Dis_{T}}[2n].
\end{equation}
\Cref{lem:complexity} states that $\Pi_{\Dis_{T}}[2n] \leq (2n)^{d(T+1)}$ which combined with our choice of $\eps$, gives:
\[2^{-\eps n}\Pi_{\Dis_{T}}[2n] \leq 2^{-\eps n}(2n)^{d(T+1)}=\delta.\]
Hence, \cref{eq:c0v} implies \cref{eq:toshow987} because $\dis_{h,\bc} \in \Pi_{\Dis_{T}}$.
\end{proof}

We now prove \Cref{lem:lol2}.
\begin{proof}[Proof of \Cref{lem:lol2}]
Note that for $\dis_{h,\bc}$ defined as in \cref{eq:PhifhT}, $\dis_{h, \bc}(x)=1$ if and only if $\wwtp$ rejects $x$ when the algorithm's choices are $h\in C$ and $\bc\in C^T$, i.e., $x\not\in S$. Thus,
\[\rej_\bx(S)=\frac{1}{n} \sum_{i=1}^n \dis_{h,\bc}(x_i) \text{~ and ~} \rej_\bz(S) = \frac{1}{n} \sum_{i=1}^n \dis_{h,\bc}(z_i).\]
Hence, it suffices to show,
\begin{equation}\label{eq:toshow9876}
\Pr_{\bx,\bz\sim P^n}\left[\exists \bc \in C^T, h \in C:~\frac{1}{n}\sum_{i=1}^n \dis_{h,\bc}(z_i)>\frac{2}{n}\sum_{i=1}^n \dis_{h,\bc}(z_i)+ \eps\right]\leq \delta
\end{equation}
\Cref{lemma:trans:aux} (with $\alpha=1$) implies that:
\[\Pr_{\bx, \bz \sim P^n}\left[\exists \phi \in \Dis_{T}:\left(\sum_i \phi(x_i) = 0\right) \wedge \left(\sum_i \phi(z_i) \geq \eps n\right)  \right] \leq e^{-\frac{2}{9}\eps n}\Pi_{\Dis_{T}}[2n].
\]
\Cref{lem:complexity} states that $\Pi_{\Dis_{T}}[2n] \leq (2n)^{d(T+1)}$ which combined with our choice of $\eps$, gives:
\[e^{-\frac{2}{9}\eps n}\Pi_{\Dis_{T}}[2n] \leq e^{-\frac{2}{9}\eps n}(2n)^{d(T+1)}=\delta.\]
Hence, the above implies \cref{eq:toshow9876} because $\dis_{h,\bc} \in \Pi_{\Dis_{T}}$.
\end{proof}

We now prove \Cref{lem:omg}.
\begin{proof}[Proof of \Cref{lem:omg}]
Fix $f,h\in\functions$.
For every $\bc\in C^T$, let $S=S(h,\bc)$
and define:
\[g_{\bc}(x)\defeq\begin{cases}1 & \text{ if } f(x)\neq h(x) \wedge x \in S\\0 &
\text{otherwise}.\end{cases} \text{ ~and~ }G\defeq\{g_{\bc}:~\bc\in C^T\}\]
So $G$ depends on $h,f$ which we have fixed. Note that
$g_{\bc}(x)=1$ iff $h|_S(x)=1-f(x)$. Hence,
\[\frac{1}{n}\sum_{i\in [n]} g_{\bc}(\tx_i)=\err_\tbx(h|_S, f).\]

\Cref{eqn:Blumer2} of \Cref{lemma:gen:aux} (with $\gamma=1/2$ and substituting $Q$ for $P$ and $\eps'=2\eps$ for $\eps$) implies that for the entire class of functions $G$:
\[\Pr_{\tbx\sim Q^n}\left[\exists g \in G: \left(\E_{x'\sim Q}[g(x')]>2\eps\right)\wedge\left(\frac{1}{n}\sum_{i\in [n]} g(\tx_i)\leq \eps\right)\right]\leq 8 \Pi_G[2n]e^{-\eps n/8}.
\]
By definition of~$G$, the above implies that,
\[\Pr_{\tbx\sim Q^n}\left[\exists \bc\in C^T:~\left(\err_Q(h|_{S(h,\bc)},f) >2\eps\right)~\wedge~\left(\err_{\tbx}(h|_{S(h,\bc)},f)\leq \eps\right)\right]\leq 8 \Pi_{G}[2n]e^{-\eps n/8}.
\]
Thus, it remains to prove that
\[
8 \Pi_{G}[2n]e^{-\eps n/8}\leq \delta.
\]
To bound $\Pi_G[2n]$, note that $g_\bc(x)= 1-\dis_{\barf,h,\bc}(x)$ and thus $\Pi_G[2n]=\Pi_{\Dis_{\barf, h, T}}[2n]$, which is at most $(2n)^{dT}$ by \Cref{lem:complexity}. Since $T\leq 1/\eps$:
\[ 8(2n)^{dT}e^{-\eps n/8} \leq 8(2n)^{d/\eps}e^{-\eps n/8}.\]
Hence it suffices to show that the above is at most $\delta$ for $\eps\geq \frac{8 \ln 8/\delta}{n} + \sqrt{\frac{8d \ln 2n}{n}}$, or equivalently that,
\[
\eps \frac{n}{8} - \frac{d}{\eps} \ln 2n \geq \ln \frac{8}{\delta}.
\]
By multiplying both sides of the equation by $\eps\cdot \frac{8}{n}$ we get
\[
\eps^2  - \frac{8}{n}d \ln 2n \geq \eps \frac{8}{n} \ln \frac{8}{\delta}.
\]
Substituting $c=\frac{8d \ln 2n}{n}$ and $b=\frac{8 \ln 8/\delta}{n}$, we have $\eps\geq b+\sqrt{c}$, and what we need to show above is equivalent to:
\[
\eps^2 - c \geq \eps b
\]
or equivalently
\[
\eps(\eps-b) \geq c
\]
which holds for any $\eps \geq b+ \sqrt{c}$ because the LHS above is $\geq (b+\sqrt{c})\sqrt{c} \geq c$.
\end{proof}

We next prove \Cref{lem:omg2}.
\begin{proof}[Proof of \Cref{lem:omg2}]
Fix any $T\geq 1$ and any $h\in \functions$.
Consider $\dis_{h,\bc}$ as defined in \cref{eq:phi}, where $\dis_{h,\bc}(x)=1$ iff $x \not\in S(h, \bc)$ is rejected. Thus, \[\rej_\bx(S(h,\bc))=\frac{1}{n}\sum_{i=1}^n \dis_{h,\bc}(x_i)~\mbox{ and }~\rej_P(S(h, \bc)) = \E_{x'\sim P}[\dis_{h,\bc}(x')].\]
By Lemma~\ref{lemma:gen:aux} (Equation~\eqref{eqn:Blumer1}), the probability that any such $\dis_{h,\bc} \in \Dis_{T}$ is 0 on all of $\bx$ but has expectation on $P$ greater than $\xi=\frac{2}{n}(d(T+1)\ln(2n)+\ln \frac{2}{\delta})$ is at most:
\[2 \Pi_{\Dis_{T}}[2n]2^{-\xi n/2} \leq 2 (2n)^{d(T+1)} 2^{-\xi n/2}=\delta,\]
where the first inequality follows from the fact that  $\Pi_{\Dis_{T}}[2n] \leq (2n)^{d(T+1)}$, which follows from \Cref{lem:complexity}. Similarly, \cref{eqn:Blumer2} of \Cref{lemma:gen:aux} (with $\gamma=1/2$ and $\eps=2\alpha$) implies that:
\[\Pr_{\bx\sim P^n}\left[\exists h,\bc: \left(\E_{x' \sim P}[ \dis_{h,\bc}(x')]>2\alpha\right)\wedge\left(\frac{1}{n}\sum_{i \in [n]} \dis_{h,\bc}(x_i)\leq \alpha\right)\right]\leq 8 \Pi_{\Dis_{T}}[2n]e^{-\alpha n/8}.
\]
For $\alpha $ as in the lemma, the right hand side above is at most $\delta$.
\end{proof}

\section{Proofs of lower bounds}\label{ap:lower-proofs}
We note that, in the lower bound of \Cref{thm:lower-gen}, the distribution $Q$ is fixed, independent of $f$. Since $Q$ is used only for unlabeled test samples, the learning algorithm can gain no information about $Q$ even if it is given a large number $m$ of test samples. In particular, it implies that even if one has $n$ training samples and  infinitely many samples from $Q$, one cannot achieve error less than $\Omega(\sqrt{d/n})$. It would be interesting to try to improve the lower-bound to have a specific dependence on $m$ (getting $\Omega(\sqrt{1/n}+1/m)$ is likely possible using a construction similar to the one below). Also, the lower-bound could be improved if one had fixed distributions $\nu, P, Q$ independent of $n$.
\begin{proof}[Proof of \Cref{thm:lower-gen}]
Let $X=\nats$ and $C$ be the concept class of functions which are 1 on exactly $d$ integers, which can easily be seen to have VC dimension $d$. The distribution $P$ is simply uniform over $[8n]=\{1,2,\ldots, 8n\}$. 
Let $k=\sqrt{8dn}$. 
The distribution $Q$ is uniform over $[k]$.
We consider a distribution $\nu$ over functions $f\in C$ that is uniform over the $k \choose d$ functions that are 1 on exactly $d$ points in $[k]$. We will show,
\begin{equation}\label{eq:ado}
\E\nolimits_{f \sim \nu}\left[\E\nolimits_{\substack{\bx \sim P^n\\ \tbx \sim Q^n}} \left[ \rej_{P} + \err_{Q} \right]\right] \geq K \sqrt{\frac{d}{n}}.
\end{equation}
By the probabilistic method, this would imply the lemma.

The set of training samples is $T = \{x_i: i \in [n]\} \subseteq [8n]$.
Say an $j \in [k]$ is ``unseen'' if it does not occur as a training example, $j \not\in T$. WLOG, we may assume that the learner makes the same classification $h|_S$ for each unseen $j \in [k]$ since an asymmetric learner can only be improved by making the (same) optimal decision for each unseen $j \in [k]$, where the optimal decisions are defined to be those that minimize $\E[\rej_P+\err_Q\mid \bx, f(\bx)]$. (The unlabeled test are irrelevant because $Q$ is fixed.)

Now, let $U\leq k$ be the random variable that is the number of seen $j \in [k]$ and $V \leq d$ be the number that are labeled 1 (which the learner can easily determine).
\begin{align*}
    U &= |T  \cap [k]|\\
    V &= |\{j \in T \cap [k] : f(j)=1\}|.
\end{align*}
Note that $\E[U]\leq k/8$ and $\E[V]\leq d/8$ since each $j \in [k]$ is observed with probability $\leq 1/8$ by choice of $P$ (the precise observation probability is $1-(1-\frac{1}{8n})^n \leq \frac{1}{8}$). These two inequalities implies that,
\[\E\left[\frac{U}{k}+\frac{V}{d}\right]\leq \frac{1}{8}+\frac{1}{8}=\frac{1}{4}.\]
Thus, by Markov's inequality,
\[\Pr\left[\frac{U}{k}+\frac{V}{d} \leq \frac{1}{2}\right]\geq \frac{1}{2}.\]
This implies that, with probability $\geq 1/2$, both $U \leq k/2$ \textit{and} $V \leq d/2$. Suppose this event happens. Now, consider three cases.

Case 1) if the learner predicts $\question$ on all unseen $j \in [k]$, then \[\rej_P\geq \frac{k}{2}\cdot\frac{1}{8n}=\sqrt{\frac{d}{32n}}\] because there are at least $k/2$ unseen $j \in [k]$ and each has probability $\frac{1}{8n}$ under $P$.

Case 2) if the learner predicts 0 on all unseen $j \in [k]$, then
\[\err_Q \geq \frac{d}{2}\cdot \frac{1}{k}=\sqrt{\frac{d}{32n}},\] because there are at least $d/2$ 1's that are unseen and each has probability $1/k$ under $Q$.

Case 3) if the learner predicts 1 on all unseen $j \in[k]$ then \[\err_Q \geq \left(\frac{k}{2}-d\right)\frac{1}{k}=\frac{1}{2}-\sqrt{\frac{d}{8n}}\geq \sqrt{\frac{d}{8n}} > \sqrt{\frac{d}{32n}}\] because there are at least $k/2-d$ unseen 0's, each with probability $1/k$ under $Q$ (and by assumption $n \geq 2d$ so $\sqrt{d/(8n)} \leq 1/4$). Thus in all three cases,
$\rej_P + \err_Q \geq \sqrt{d/(32n)}$. Hence,
\[\E\bigl[\rej_P + \err_Q\mid U\leq k/2, V \leq d/2\bigr]\geq \sqrt{\frac{d}{32n}}\]
Since $U \leq k/2, V\leq d/2$ happens with probability $\geq 1/2$, we have that $\E[\rej_P + \err_Q] \geq \frac{1}{2}\sqrt{d/(32n)}$ as required. This establishes \cref{eq:ado}.
\end{proof}
We now prove our agnostic lower bound.
\begin{proof}[Proof of \Cref{ag-lower}]
Let $X=\nats$ and $C$ consist of the singleton functions that are 1 at one integer and 0 elsewhere. The VC dimension of $C$ is easily seen to be 1.

Consider first the case in which $\tilde\eta \geq \sqrt{\eta/8}$. In this case, we must construct distributions $\mu, \tmu$ and $f\in C$ such that,  $\E[\err_{\tmu}(h|_S)+\rej_\mu(S)] \geq \tilde\eta$. This is trivial: let $\mu$ be arbitrary and $\tmu(1,1)=\tilde\eta$ and $\tmu(1,0)=1-\tilde\eta$. It is easy to see that no classifier has error less than $\tilde\eta$ since $\tilde\eta \leq 1/2$.

Thus it suffices to give $\mu, \tmu$ and $f\in C$ such that, $\err_\tmu(f)=0$, $\err_\mu(f)= \eta$, and,
\begin{equation}\label{eq:toshowlower}
\E_{\substack{(\bx,\by)\sim \mu^n\\(\tbx,\tilde\by)\sim\tmu^n}}[\err_{\tmu}(h|_S)+\rej_\mu(S)] \geq \sqrt{\eta/8}.
\end{equation}
In particular, we will give a distribution over $f, \mu, \tilde{\mu}$ for which the above holds for the output $h|_S$ of any learning algorithm. By the probabilistic method, this implies that for each learning algorithm, there is at least $f, \mu, \tmu$ for which \cref{eq:toshowlower} holds.
To this end, let $k=\lfloor \sqrt{2/\eta} \rfloor$. Let $\mu$ be the distribution which has $\mu(x,0)=\eta/2$ for $x \in [k]$ and $\mu(k+1,0)=1-k \eta/2$, so $\mu$ has $y=0$ with probability 1. Let $f$ be 1 for a uniformly random $x^* \in [k]$ so $\err_\mu(f)=\eta/2$. Let $\tmu$ be the distribution where $\tmu(x, f(x))=1/k$ for $x \in [k]$, so $x$ is uniform over $[k]$ with $\err_{\tmu}(f)=0$.

Now, given the above distribution over $f, \mu, \tmu$, there is an optimal learning algorithm that minimizes $\E[\err_{\tmu}+\rej_\mu]$. Moreover, notice that the algorithm learns nothing about $\mu$ or $\tmu$ from the training data since $\mu$ is fixed as is the distribution over unlabeled examples. Thus the optimal learner, by symmetry, may be taken to make the same classification for all $x \in [k]$. Thus, consider three cases.
\begin{itemize}
    \item The algorithm predicts $h|_S(x)=\question$ for all $x \in [k]$. In this case, \[\rej_\mu \geq k\frac{\eta}{2} = \lfloor \sqrt{2/\eta}\rfloor \frac{\eta}{2} \geq \frac{1}{2}\sqrt{\eta/2}\] using the fact that $\lfloor r \rfloor \geq r/2$ for $r \geq 1$.
    \item The algorithm predicts $h|_S(x)=0$ for all $x \in [k]$. In this case, \[\err_\tmu = \frac{1}{k} \geq \sqrt{\eta/2}.\]
    \item The algorithm predicts $h|_S(x)=1$ for all $x \in [k]$. In this case, since $\eta \leq 1/2$, $k \geq 2$ and $\err_\tmu \geq 1/2$.
\end{itemize}
In all three cases, $\err_\tmu+\rej_\mu \geq \sqrt{\eta/8}$ proving the lemma.
\end{proof}

We now present the proof of our transductive lower bound.
\begin{proof}[Proof of \Cref{thm:lower-trans}]

Just as in the proof of \Cref{thm:lower-gen}, let $X=\nats$ and $C$ again be the concept class of functions that have exactly $d$ 1's, which has VC dimension $d$. Again, let $P$ be the uniform distribution over $[N]$ for $N=8n$.

We will construct a distribution $\nu$ over $C$ and randomized adversary $\cA(\bx, \bz, f)$ that outputs $\tbx\in X^n$ such that, for all $L$,
\[\E[\rej_\bz + \err_\tbx] \geq \lambda,\]
where $\lambda$ is a lower bound and expectations are over $\bx \sim P^n, \bz \sim P^m$ and $f \sim \nu$. By the probabilistic method again, such a guarantee implies that for any learner $L$, there exists some $f \in C$ and deterministic adversary $\cA(\bx, \bz)$ where the above bound holds for that learner.

We will show two lower bounds that together imply the lemma. The first lower bound will follow from \Cref{thm:lower-gen} and show that,
\[\E[\rej_\bz + \err_\tbx] \geq K \sqrt{d/n},\]
where expectations are over $\bx \sim P^n, \bz \sim P^m, f \sim \nu$. Here $K$ is the constant from \Cref{thm:lower-gen}. To get this, the adversary $\cA(\bx, \bz, f)$ simply ignores the true tests $\bz$ and selects $\tbx\sim Q^m$. By linearity of expectation, for any learner, $\E[\rej_\bz]=\E[\rej_P]$ and $\E[\err_\tbx]=\E[\err_Q]$.

It remains to show a distribution $\nu$ over $C$ and adversary $A$ such that, for all learners,
\begin{equation}\label{eq:2s2}\E[\rej_\bz + \err_\tbx] \geq K \sqrt{d/m},\end{equation}
for some constant $K$ and $m<n$ (for $m\geq n$, the previous lower bound subsumes this). Let $\nu$ be the uniform distribution over those $f \in C$ that have all $d$ 1's in $[N]$, i.e., uniform over $\{f \in C:\sum_{i \in [N]}f(i)=d\}$.

Let $A\defeq \{x \in [N]: f(x)=0\}$ and $B\defeq \{x \in \nats: f(x)=1\}$ so $|A|=N-d$ and $|B|=d$.

Let $a=\lfloor \sqrt{md} \rfloor$ and $b = \lceil d/2 \rceil$, and $r = \lfloor m/(a+b) \rfloor$.
The adversary will try to construct a dataset $\tbx$ with the following properties:
\begin{itemize}
    \item $\tbx$ contains exactly $a$ distinct $\tx\in A$ and each has exactly $r$ copies. (Since $a\leq m < N-d$, this is possible.)
    \item There are exactly $b$ distinct $\tx \in B$ and each has exactly $r$ copies.
    \item The remaining $m-r(a+b)$ examples are all at $\tx=N+1$ (these are ``easy'' as the learner can just label them 0 if it chooses).
\end{itemize}

We say $x$ is \textit{seen} if $x \in \bx$ (this notation indicates $x \in \{x_i: i \in [n]\}$ did not occur in the training set) and \textit{unseen} otherwise. Now, we first observe that with probability $\geq 1/8$, the following event $E$ happens: there are at most $d-b$ seen 1's ($x_i \in B$) in the training set and there are at least $a$ distinct unseen 0's in the true test set $\bz$, i.e.,
\begin{align*}
    V_1\defeq |\{i \in [n]: x_i \in B\}| &\leq d-b\\
    V_0\defeq |\{z \in A: (z \in \bz) \wedge (z \not\in \bx)\}| & \geq a
\end{align*}
Note that $\E[V_1]= d n/N= d/8$. Markov's inequality guarantees that with probability $\geq 3/4$, $V_1 \leq d/2$ (otherwise $\E[V_1] > d/8$). Since $V_1$ is integer, this means that with probability $\geq 3/4$, $V_1 \leq \lfloor d/2 \rfloor = d-b$. Similarly, for any $i \in A$, the probability that it occurs in $\bz$ and not in $\bx$ is,
\[\left(1-\frac{1}{N}\right)^n\left(1-\left(1-\frac{1}{N}\right)^m\right) \geq \left(1-\frac{n}{N}\right)\left(1-e^{-\frac{m}{N}}\right) \geq  \frac{7}{8} \cdot \frac{15}{16}\frac{m}{N} \geq 0.8\frac{m}{N},\]
where in the above we have used the fact that $(1-t)\leq e^{-t}$ for $t > 0$ and $1-e^{-t}\geq (15/16)t$ for $t \leq 1/8$.
Hence, since $|A|=N-d$,
\[
\E[V_0] \geq (N-d)0.8\frac{m}{N} \geq \left(\frac{7}{8} N\right)0.8\frac{m}{N}=0.7m \geq 0.7a.\]
In particular, Markov's inequality implies that with probability at least 0.4, $V_0 \geq 0.5 m$ (otherwise $\E[V_0] < 0.6 (0.5 m) + 0.4 m = 0.7m$). Thus, with probability $\geq 1-1/4-0.6 \geq 1/8$.

If this event $E$ does not happen, then the adversary will take all $\tx = N+1$, making learning easy.  However, if $E$ does happen, then there must be at least $a$ unseen 0's in $\bz$ and $b$ unseen 1's and the the adversary will select $a$ random unseen 0's from $\bz$ and $b$ random unseen 1's, uniformly at random. It will repeat these examples $r$ times each, add $m-r(a+b)$ copies of $\tx=N+1$, and permute the $m$ examples.

Now that the adversary and $\nu$ have been specified, we can consider a learner $L$ that minimizes the objective $\E[\rej_\bz + \err_\tbx]$. Clearly this learner may reject $N+1 \not\in S$ as this cannot increase the objective. Now, by symmetry the learner may also be assumed to make the same classification on all $r(a+b)$ examples $\tx \in [N]$ as these examples are all unseen and indistinguishable since $B$ is uniformly random.

Case 1) If $h|_S(\tbx_i)=\question$ for all $i$ then
\[\rej_\bz = \frac{a}{m} = \frac{\lfloor \sqrt{md} \rfloor}{m} \geq  \frac{\sqrt{md}/2}{m} = \frac{1}{2}\sqrt{\frac{d}{m}},\]
using the fact that $a \geq \sqrt{md}/2$ because
$a \geq \sqrt{md}/2$ since $\lfloor t \rfloor \geq t/2$ for $t \geq 1$.

Case 2) If $h|_S(\tbx_i)=0$ for all $i$ then,
\[\err_\tbx = \frac{b r}{m} \geq \frac{b\sqrt{m/d}}{4m} = \frac{b}{4\sqrt{md}} \geq \frac{d}{8\sqrt{md}}=\frac{1}{8}\sqrt{\frac{d}{m}}\]
In the above we have used the fact $b \geq d/2$ and that $r \geq \frac{1}{4}\sqrt{m/d},$ which can be verified by noting that:
\[\frac{m}{a+b} \geq \frac{m}{2a} \geq \frac{m}{2\sqrt{md}} =\frac{1}{2}\sqrt{\frac{m}{d}}\geq 1\] and hence $r \geq \lfloor m/(a+b) \rfloor \geq \frac{1}{2}m/(a+b) \geq \frac{1}{4}\sqrt{m/d}$ again since $\lfloor t \rfloor \geq t/2$ for $t \geq 1$.

Case 3) If $h|_S(\tbx_i)=0$ for all $i$ then, since $b \leq a$
\[\err_\tbx = \frac{b}{a+b} \geq \frac{1}{2}.\]
In all three cases, we have,
\[\rej_\bz + \err_\tbx \geq \frac{1}{8}\sqrt{\frac{d}{m}}.\]
Since $E$ happens with probability $\geq 1/2$, we have,
\[\E[\rej_\bz + \err_\tbx] \geq \Pr[E]\E[\rej_\bz + \err_\tbx\mid E] \geq \frac{1}{8} \cdot  \frac{1}{8}\sqrt{\frac{d}{m}}.\]
This is what was required for \cref{eq:2s2}.
\end{proof}

\section{Tight bounds relating train and test rejections}\label{ap:tight}

We now move on to tightly relating test and training rejections. As motivation, note that if one knew $P$ and $Q$, it would be natural to take $S^*\defeq \{x\in X: Q(x) \leq P(x)/\eps\}$ for some $\eps >0$. For $x \not\in S^*$, i.e., $x \in \barS^*$, $P(x) < \eps Q(x)$. This implies that $\rej_P(S^*) = P(\barS^*) < \eps$. It is also straightforward to verify that $\err_Q(h|_{S^*}) \leq  \err_P(h)/\eps$. This means that if one can find $h$ of error $\eps^2$ on $P$, e.g., using a PAC-learner, then this gives,
\[\rej_P(S^*) + \err_Q(h|_S^*) \leq 2\eps.\]
This suggests that perhaps we could try to learn $P$ and $Q$ and approximate $S^*$. Unfortunately, this is generally impossible---one cannot even distinguish the case where $P=Q$ from the case where $P$ and $Q$ have disjoint supports with fewer than $\Omega(\sqrt{|X|})$ examples.\footnote{To see this, consider the cases where $P=Q$ are both the uniform distribution over $X$ versus the case where they are each uniform over a random partition of $X$ into two sets of equal size. By the classic \textit{birthday paradox}, with $O(\sqrt{|X|})$ samples both cases will likely lead to random disjoint sets of samples.}

While we cannot learn $S^*$ in general, these sets $S^*$ do give the tightest bounds on $\rej_Q$ in terms of $\rej_P$.
\begin{lemma}
For any $S\subseteq X$ and distributions $P,Q$ over $X$ and any $\eps\geq 0$ such that $\rej_P(S) \leq \rej_P(S^*)$,
\begin{equation}
\rej_Q(S) \leq \rej_Q(S^*).
\label{eq:tight}
\end{equation}
\end{lemma}
Note that the $\rej_Q(S) \leq \rej_P(S) + |P-Q|_\TV$ bound can be much looser than the bound in the above lemma. For example, $|P-Q|_{\TV} = 0.91$ yet $\rej_Q(S^*) = 0.1$ for $X=\{0,1,\ldots, 100\}$, $P$ uniform over $\{1, \ldots, 100\}$, $Q$ uniform over $\{0,1,\ldots, 9\}$, $\rej_P(S)=0$, and $\eps=0.1$ (since  $S^*=\{1,2,\ldots, 100\}$ and only $0 \not\in S^*$). One can think of classifying images of a mushroom as ``edible'' or not based on training data of 100 species of mushrooms, with test data including one new species.
\begin{proof}
Since $\eps Q(x)-P(x)>0$ iff $x \not\in S^*$,
\begin{align*}
    \eps \rej_Q(S^*)-\rej_P(S^*)  &= \sum_{x \not\in S^*} \eps Q(x)-P(x)  \\
    &\geq\sum_{x \not\in S} \eps Q(x)-P(x) =  \eps \rej_Q(S) -\rej_P(S) \\
\Rightarrow ~~\eps(\rej_Q(S^*) -\rej_Q(S)) &\geq \rej_P(S^*) - \rej_P(S) \geq 0.
\end{align*}
\end{proof}

\end{document}